\documentclass[twoside,11pt]{article}
\usepackage{jmlr2e,amssymb,amsmath}
\usepackage{url,psfrag,color,enumerate}

\usepackage{epsfig}
\usepackage{graphics}

\RequirePackage{natbib}

\newcommand{\mysec}[1]{Section~\ref{sec:#1}}

\newcommand{\eq}[1]{Eq.~(\ref{eq:#1})}
\newcommand{\myfig}[1]{Figure~\ref{fig:#1}}

\newcommand{\BEAS}{\begin{eqnarray*}}
\newcommand{\EEAS}{\end{eqnarray*}}
\newcommand{\BEA}{\begin{eqnarray}}
\newcommand{\EEA}{\end{eqnarray}}
\newcommand{\BEQ}{\begin{equation}}
\newcommand{\EEQ}{\end{equation}}
\newcommand{\BIT}{\begin{itemize}}
\newcommand{\EIT}{\end{itemize}}
\newcommand{\BNUM}{\begin{enumerate}}
\newcommand{\ENUM}{\end{enumerate}}
\newcommand{\BA}{\begin{array}}
\newcommand{\EA}{\end{array}}
\newcommand{\diag}{\mathop{\rm diag}}
\newcommand{\var}{\mathop{\rm var}}
\newcommand{\rank}{\mathop{\rm rank}}
\newcommand{\vect}{\mathop{\rm vec}}

\newcommand{\Diag}{\mathop{\rm Diag}}

\newcommand{\rb}{\mathbb{R}}

\newcommand{\tr}{{\rm tr}}
\newcommand{\idm}{\mathbf{I} }

\def \E{{\mathbb E}}
\def \P{{\mathbb P}}

\def \hS{\hat{\S}}
\def \hD{\hat{\D}}
\def \J { \mathbf{J} }
\def \S{ \Sigma }
\def \D{ \Delta }
\def \r{ \mathbf{r}}

\def \W{ \mathbf{W} }
\def \w{ \mathbf{w} }

\def \U{ \mathbf{U} }
\def \s{ \mathbf{s} }
\def \V{ \mathbf{V} }

\def \sig{ \sigma^2}

\newcounter{hyp}
\newenvironment{hyp}{\refstepcounter{hyp}\begin{itemize}
  \item[({\bf{A}\arabic{hyp}})]}{\end{itemize}}
  
\newcommand{\hypref}[1]{({\bf{A}\ref{hyp:#1}})}

\title{ Consistency of trace norm minimization\\[1cm]}

\author{\name Francis R. Bach \email francis.bach@mines.org\\
\addr
INRIA - Willow project\\
D\'epartement d'Informatique, Ecole Normale Sup\'erieure\\
45, rue d'Ulm\\
75230 Paris, France\\
       }

\begin{document}
\maketitle

\vspace*{1cm}

\begin{abstract}
Regularization by the sum of singular values, also referred to as the  \emph{trace norm}, is a popular technique for estimating 
low rank rectangular matrices. In this paper, we extend some of the consistency results of the Lasso to provide necessary and sufficient conditions for rank consistency of trace norm minimization with the square loss. 
We also 
provide an adaptive version that is rank consistent even when the necessary condition
for the non adaptive version is not fulfilled. 

\end{abstract}
 
\vspace*{1cm}

\section{Introduction}
In recent years, regularization by various non Euclidean 
norms has seen considerable interest. In particular, in the context of linear supervised learning, norms such as the $\ell_1$-norm may induce sparse loading vectors, i.e., loading vectors with low cardinality or $\ell_0$-norm. Such regularization schemes, also known as the Lasso~\citep{lasso} for least-square regression, come with efficient path following algorithms~\citep{lars}. Moreover, recent work has studied conditions under which such procedures consistently estimate the sparsity pattern of the loading vector~\citep{yuanlin,Zhaoyu,zou}. 
 
 When learning on rectangular matrices, the rank is a natural extension of the cardinality, and the sum of singular values, also known as the trace norm or the nuclear norm, is the natural extension of the $\ell_1$-norm; indeed, as the $\ell_1$-norm is the convex envelope of the $\ell_0$-norm on the unit ball (i.e., the largest lower bounding convex function)~\citep{boyd}, the trace norm is the convex envelope of the rank over the unit ball of the spectral norm~\citep{fazel}. In practice, it leads to low rank solutions~\citep{fazel,Srebro2005Maximum} and has seen recent increased interest in the
 context of collaborative filtering~\citep{Srebro2005Maximum}, multi-task learning~\citep{lowrank,massi} or classification with multiple classes~\citep{srebro-mc}.
 
 In this paper, we consider the rank consistency of trace norm regularization with the square loss, i.e., if the data were actually generated by a low-rank matrix, will the matrix and its rank be consistently estimated? In \mysec{consistency}, we provide necessary and sufficient conditions for the rank consistency that are extensions of corresponding results for the Lasso~\citep{yuanlin,Zhaoyu,zou} 
 and the group Lasso~\citep{grouplasso}. We do so under two sets of sampling assumptions detailed in \mysec{assumptions}: a full i.i.d assumption and a non i.i.d assumption which is natural in the context of collaborative filtering. 
 
As for the Lasso and the group Lasso, the necessary condition implies that such procedures do not always estimate the rank correctly; following the adaptive version of the Lasso and group Lasso~\citep{zou}, we design an adaptive version to achieve $n^{-1/2}$-consistency and rank consistency, with no consistency conditions.
 Finally,  in \mysec{algorithms}, we present a smoothing approach to convex optimization with
 the trace norm, while  in \mysec{simulations}, we show simulations on toy examples to illustrate the consistency results.

\section{Notations}
In this paper we consider various norms on vectors and matrices. On vectors $x$ in 
$\rb^d$, we
always consider the Euclidean norm, i.e., $\| x \| = ( x^\top x)^{1/2}$.
On rectangular matrices in $  \rb^{ p \times q}$, however, we consider several norms, based on singular values~\citep{sun}:
the \emph{spectral norm} $\| M \|_2 $ is the largest singular value (defined
 as $\|M\|_2 = \sup_{ x \in \rb^q } \frac{ \| M x\| }{\|x\|}$), 
 the \emph{trace norm} (or \emph{nuclear norm})
  $\| M \|_\ast$ is the sum of singular values,
and the \emph{Frobenius norm} $\| M \|_F $ is the $\ell_2$-norm of singular values (also defined
as $\|M\|_F = ( \tr M^\top M)^{1/2}$).  In Appendix~\ref{app:svd} and~\ref{app:tracenorm}, we review and derive relevant tools and results regarding perturbation of singular values as well as the trace norm.

%We also use the \emph{uniform} norm $\|x\|_\infty = \max_{i=1,\dots,d} | x_i|$, extended to matrices by vectorization, i.e., $\| M\|_\infty = \| \vect(M) \|_\infty$ is the largest magnitude of the elements of $M \in \rb^{p \times q}$ (and not the operator norm associated to the uniform norms in $\rb^p$ and $\rb^q$).

Given a matrix $M \in \rb^{ p \times    q}$, $\vect(M)$ denotes the vector
in $\rb^{ pq }$ obtained by stacking its columns into a single vector; and $A \otimes B$
denotes the Kronecker product between matrices $A \in
\rb^{p_1 \times q_1}$ and $B \in
\rb^{p_2 \times q_2} $, defined as the matrix in $
\rb^{p_1 p_2 \times q_1 q_2}$, defined by blocks of sizes
$p_2 \times q_2$ equal to $a_{ij} B $. We make constant use of the following identities:
$(B^\top \otimes A) \vect(X) = \vect( AXB)$ 
and $\vect(uv^\top) = v \otimes u$. For more details and properties,
see~\citet{golub83matrix} and~\citet{magnus}. We also use the notation $\Sigma W$ for $\Sigma \in \rb^{ pq \times pq}$ and
$W \in \rb^{p \times q}$ to design the matrix in $\rb^{p \times q}$ such that
$\vect( \Sigma W) = \Sigma \vect(W)$ (note the potential confusion with $\S W$ when $\S$ is a matrix with $p$ columns). 
 
We also use the following standard asymptotic notations: 
a random variable $Z_n$ is said to be of order $O_p(a_n)$ if for
any $\eta>0$, there exists $M>0$ such that $\sup_n P( |Z_n| > M a_n ) < \eta$. Moreover, 
$Z_n$ is said to be of order $o_p(a_n)$ if $Z_n/a_n$ converges to zero in probability, i.e.,
if for
any $\eta>0$,   $ P( |Z_n| \geqslant \eta a_n ) $ converges to zero.
See~\citet{VanDerVaart} and \citet{shao}
for further definitions and properties of asymptotics in probability.

Finally, we use the following two conventions: lowercase for vectors and uppercase for matrices, while bold fonts are reserved for population quantities.

\section{Trace norm minimization}
\label{sec:regular}
We consider the problem of predicting a real random variable $z$ as a linear function
of a matrix $M \in \rb^{p \times q}$, where $p$ and $q$ are two
fixed strictly positive integers. Throughout this paper, we assume that 
we are given $n$ observations $(M_i,z_i)$, $i=1,\dots,n$, and we consider the following optimization problem with the square loss:
\BEQ
\label{eq:problemM}
\min_{W \in \rb^{p \times q} } \frac{1}{2n} \sum_{i=1}^n ( z_i - \tr W^\top M_i )^2 + \lambda_n \| W \|_\ast,
\EEQ
where $\|W\|_\ast$ denotes the \emph{trace norm} of $W$.

\subsection{Special cases}
Regularization by the trace norm has numerous applications (see, e.g.,~\citet{fazel2} for a review); in this paper, we are particularly interested in the following two situations:

\paragraph{Lasso and group Lasso}
When $x_i \in \rb^m$, we can define $ M_i = \Diag(x_i) \in \rb^{m \times m}$ as the diagonal matrix with $x_i$ on the diagonal. In this situation the minimization
of problem  \eq{problemM} must lead to diagonal solutions (indeed the minimum trace norm matrix with fixed diagonal is the corresponding
diagonal matrix, which is a consequence of Lemma~\ref{lemma:fenchel} and
Proposition~\ref{prop:subdiff}) and for a diagonal matrix the trace norm is simply the $\ell_1$ norm of the diagonal. Once we have derived our consistency conditions, we
check in \mysec{lasso}  that they 
actually lead to the known ones for the Lasso~\citep{yuanlin,Zhaoyu,zou}.

We can also see the group Lasso as a special case; indeed, if $x_{ij} \in \rb^{d_j}$ for $j=1,\dots,m$, $i=1,\dots,n$, then we define $M_i \in \rb^{(\sum_{j=1}^m d_j) \times m}$ as the block diagonal matrix
(with non square blocks) with diagonal blocks $x_{ji}$, $j=1,\dots,m$. Similarly, the optimal $\hat{W}$ must share the same block-diagonal form, and its singular values are exactly the norms of each block, i.e., the trace norm is indeed the sum of the norms of each group. We also get back results from~\citet{grouplasso} in \mysec{lasso}.

Note that the Lasso and group Lasso can be seen as special cases where the singular vectors are fixed. However, the main difficulty in analyzing trace norm regularization, as well as the main reason for it use, is that singular vectors are not fixed and those can often be seen as implicit features learned by the estimation procedure~\citep{Srebro2005Maximum}. In this paper we derive consistency results about the value and numbers of such features.

\paragraph{Collaborative filtering and low-rank completion}
Another natural application is collaborative filtering where two types of attributes $x$ and $y$ are observed and we consider bilinear forms in $x$ and $y$, which can be written as a linear form in $M=xy^\top$ (thus it corresponds to situations where all matrices $M_i$ have rank one). In this setting, the matrices $M_i$ are not usually i.i.d. but exhibit a statistical dependence structure outlined in \mysec{assumptions}. A special case here is when then no attributes are observed and we simply wish to complete a partially observed matrix~\citep{Srebro2005Maximum,lowrank}. The results presented in this paper do not immediately apply because the dimension of the estimated matrix may grow with the number of observed entries and this situation is out of the scope of this paper.

\subsection{Assumptions}

\label{sec:assumptions}
We   make the following assumptions
 on the sampling distributions of $M \in \rb^{ p \times q}$ for the problem in
\eq{problemM}. We let denote:
$\hS_{mm} = \frac{1}{n} \sum_{i=1}^n \vect(M_i) \vect(M_i)^\top \in \rb^{ p q \times p q}$, and we consider the following assumptions:

\begin{hyp}
\label{hyp:model}
Given $M_i$, $i=1,\dots,n$, the $n$ values $z_i$ are i.i.d. and there exists
$\W \in \rb^{p \times q}$ such that for all $i$, $ \E(z_i|M_1,\dots,M_n) = \tr \W^\top M_i$
 and $\var (z_i | M_1,\dots,M_n)$ is a strictly positive constant $\sig$. $\W$ is not equal to zero and does not have full rank. \end{hyp}

\begin{hyp}
\label{hyp:samplingM}
There exists an \emph{invertible} matrix $\S_{mm} \in \rb^{ p q \times p q}$ such that 
$ \E \| \hS_{mm} - \S_{mm} \|_F^2 = O(\zeta_n^2)$ for a certain sequence $\zeta_n$ that tends to zero.
\end{hyp}

\begin{hyp}
\label{hyp:normalityM}   The random variable
${n^{-1/2}}  \sum_{i=1}^n \varepsilon_i \vect(M_i )$ is converging in distribution to a normal distribution with
mean zero and covariance matrix $\sig \S_{mm}$.
\end{hyp}

Assumption \hypref{model} states that given the input matrices $M_i$, $i=1,\dots,n$ we have a linear
prediction model, where the loading matrix $\W$ is non trivial and rank-deficient, the goal being to estimate this rank (as well as the matrix itself).  We let denote 
 $\W = \U \Diag(\s) \V^\top$ its singular value decomposition,
 with $\U \in \rb^{ p \times \r}$ ,  $\V \in \rb^{ q \times \r}$, and $\r \in (0,\min\{p,q\})$  denotes the rank of $\W$.   We also let denote $\U_\bot \in \rb^{p \times (p-\r)}$ and $\V_\bot
\in \rb^{q \times (q-\r)}$ any orthogonal complements of $\U$ and $\V$.

We let denote $\varepsilon_i = z_i - \tr \W^\top M_i$  and
$
\hS_{Mz} = \frac{1}{n} \sum_{i=1}^n z_i M_i  \in \rb^{ p \times q}
$, $
\hS_{M\varepsilon} = \frac{1}{n} \sum_{i=1}^n \varepsilon_i M_i  
= \hS_{Mz} - \hS_{mm} \W \in \rb^{ p \times q}
$.  
We may then rewrite
 \eq{problemM} as 
\BEQ
\min_{W \in \rb^{p \times q} }  
  \frac{1}{2} \vect(W)^\top \hS_{mm} \vect(W) - \tr W^\top \hS_{Mz} + \lambda_n \| W \|_\ast,
\EEQ
or, equivalently,
\BEQ
\label{eq:problemM2bis}
\min_{W \in \rb^{p \times q} }  \frac{1}{2} \vect(W - \W)^\top \hS_{mm} \vect(W - \W) - \tr W^\top \hS_{M\varepsilon} + \lambda_n \| W \|_\ast.
\EEQ

 The sampling assumptions \hypref{samplingM} and \hypref{normalityM} may seem restrictive, but they are satisfied in the following two natural situations. The first situation corresponds to a classical full i.i.d problem, where the pairs $(z_i,M_i)$ are sampled i.i.d:
\begin{lemma}
Assume \hypref{model}. If the matrices $M_i$ are sampled i.i.d., $z$ and $M$ have finite
fourth order moments, and  $\E \left\{ \vect(M) \vect(M)^\top \right\}$
is invertible, then \hypref{samplingM} and \hypref{normalityM} are satisfied with $\zeta_n
=n^{-1/2}$.
\end{lemma}
Note the further refinement when for each $i$,  $M_i = x_i y_i^\top$ and $x_i$ and $y_i$ are independent, which implies that $\S_{mm}$ is factorized as a Kronecker product, of the form $\S_{yy} \otimes \S_{xx}$ where $\S_{xx}$ and $\S_{yy}$ are the (invertible) second order moment matrices of $x$ and $y$.

The second situation corresponds to a collaborative filtering situation where two types of attributes are observed, e.g., $x$ and $y$, and for every pair $(x,y)$ we wish to predict $z$ as a bilinear form in
$x$ and $y$: we first sample $n_x$ values for $x$, and $n_y$ values for $y$, and we select uniformly at random a subset of $n \leqslant n_x n_y$ observations from the $n_x n_y$ possible pairs. The following lemma, proved in Appendix~\ref{app:cf}, shows that this set-up satisfies our assumptions:
\begin{lemma}
\label{lemma:cf}
Assume  \hypref{model}. Assume moreover that $n_x$ values 
$\tilde{x}_1,\dots,\tilde{x}_{n_x}$ are sampled i.i.d and $n_y$ 
values $\tilde{y}_1,\dots,\tilde{y}_{n_y}$ are also sampled i.i.d. from distributions with finite fourth order moments and invertible second order moment matrices
$\S_{xx}$ and $\S_{yy}$. Assume also that a random subset of size $n$ of pairs $(i_k,j_k)$ in $ \{1,\dots,n_x\} \times \{1,\dots,n_y\}$  is sampled uniformly, then if $n_x$, $n_y$ and $n$ tend to infinity, then  \hypref{samplingM} and \hypref{normalityM} are satisfied with
$\S_{mm} = \S_{yy} \otimes \S_{xx}$ and $\zeta_n = n^{-1/2} + n_x^{-1/2} + n_y^{-1/2}$.
\end{lemma}

\subsection{Optimality conditions}
From the expression of the subdifferential of the trace norm
in Proposition~\ref{prop:subdiff} (Appendix~\ref{app:tracenorm}), we can identify the optimality condition for problem in \eq{problemM}, that  we will constantly use in the paper:
\begin{proposition}
The matrix $W$ with singular value decomposition $W=U\Diag(s)V^\top$ (with
strictly positive singular values $s$) is optimal for the problem in \eq{problemM} if and only
if
\BEQ
\label{eq:opt-vec}
 \hS_{mm} W  -  \hS_{Mz}  + \lambda_n  U V^\top  +N = 0,
 \EEQ
  with $U^\top N = 0$,
$N V=0$ and $\| N\|_2 \leqslant \lambda_n $. 
\end{proposition}
This implies notably that $W$ and $ \hS_{mm} W -   \hS_{Mz} $ have \emph{simultaneous}
singular value decompositions, and the largest singular values are less than $\lambda_n$, and exactly
equal to $\lambda_n$ for the corresponding strictly positive singular values of $W$.
Note that when all matrices are diagonal (the Lasso case), we obtain the usual optimality conditions (see also~\citet{fazel2} for further discussions).

\section{Consistency results}
\label{sec:consistency}
We consider two types of consistency, the \emph{regular consistency}, i.e., we want the probability
$\P( \| \hat{W}- \W \| \geqslant \varepsilon) $ to tend to zero as $n $ tends to infinity, for all $\varepsilon >0$. We also consider the
\emph{rank consistency}, i.e., we want that
$\P( \rank(\hat{W}) \neq \rank(\W) ) $ tends to zero as $n$ tends to infinity. Following the similar
properties for the Lasso, the consistency depends on the decay of the regularization 
parameter. Essentially, we obtain the following results:
\begin{enumerate}[\hspace*{.5cm}a)]
\item  if  $\lambda_n$ does not tend to zero, then the trace norm estimate $\hat{W}$ is not consistent; 
\item if $\lambda_n$ tends to zero faster than $n^{-1/2}$, then the estimate is consistent and its error is $O_p(n^{-1/2})$ while it is not rank-consistent with probability tending to one (see \mysec{fast});
\item  if   $\lambda_n$ tends to zero exactly at rate $n^{-1/2}$, then the estimator is consistent with error $O_p(n^{-1/2})$ but the probability of estimating the correct rank is converging to a limit in $(0,1)$ (see \mysec{sqrtn});
\item  if   $\lambda_n$ tends to zero more slowly than $n^{-1/2}$, then the estimate is consistent with 
error $O_p(\lambda_n)$ and its rank consistency depends on  specific \emph{consistency
conditions} detailed in \mysec{slow}.
\end{enumerate}
The following sections will look at each of these cases, and state precise theorems. We then consider some special cases, i.e., factored second-order moments and implications for the special cases of the Lasso and group Lasso.

The first proposition (proved in Appendix~\ref{app:asympt}) considers the case where the regularization parameter $\lambda_n$ is converging to a certain limit $\lambda_0$. When this limit is zero, we
obtain regular consistency (Corollary~\ref{prop:regular-consistency} below), while if $\lambda_0 >0$, then $\hat{W}$ tends in probability to a limit which is always different from $\W$:

\begin{proposition}
\label{prop:asympt}
Assume \hypref{model}, \hypref{samplingM} and \hypref{normalityM}. Let $\hat{W}$ be a global minimizer of \eq{problemM}.
 If $\lambda_n$ tends to a limit $\lambda_0 \geqslant 0$, then $\hat{W}$ converges in probability to the unique global minimizer of
$$
\min_{W \in \rb^{p \times q} } 
\frac{1}{2} \vect(W - \W)^\top \S_{mm} \vect(W - \W)  + \lambda_0 \| W \|_\ast.
$$
\end{proposition}
\begin{corollary}
\label{prop:regular-consistency}
Assume \hypref{model}, \hypref{samplingM} and \hypref{normalityM}.  Let $\hat{W}$ be a global minimizer of \eq{problemM}. If $\lambda_n$ tends to zero, then $\hat{W}$ converges in probability to $\W$.
\end{corollary}
We now consider finer results when $\lambda_n$ tends to zero at certain rates, slower or faster than $n^{-1/2}$, or exactly at rate $n^{-1/2}$.

\subsection{Fast decay of regularization parameter}
\label{sec:fast}
The following proposition---which is a consequence of  standard results in M-estimation~\citep{shao,VanDerVaart}---considers the case where $n^{1/2} \lambda_n$ is tending to zero, where we obtain that $\hat{W}$ is asymptotically normal with mean
$\W$ and covariance matrix $n^{-1} \sigma^2 \S_{mm}^{-1}$, i.e., for fast decays, the first order expansion is the same as the one with no regularization parameter:
\begin{proposition}
\label{prop:asympt1d}
Assume \hypref{model}, \hypref{samplingM} and \hypref{normalityM}. Let $\hat{W}$ be a global minimizer of \eq{problemM}.
 If $n^{1/2} \lambda_n$ tends to zero,  $n^{1/2}(\hat{W}-\W)$ is asymptotically normal with mean
$\W$ and covariance matrix $ \sigma^2 \S_{mm}^{-1}$.
\end{proposition}

We now consider the corresponding rank consistency results, when $\lambda_n$ goes to zero faster than $n^{-1/2}$. The following proposition (proved in Appendix~\ref{app:rank2}) states that for such regularization parameter, the solution has rank strictly greater than $\r$ with probability tending to one and can thus not be rank consistent:
\begin{proposition}
\label{prop:rank2}
Assume  \hypref{model}, \hypref{samplingM} and \hypref{normalityM}. 
If $n^{1/2} \lambda_n$ tends to zero, then $\P(\rank(\hat{W})> \rank(\W))$ tends to one.
\end{proposition}

\subsection{$n^{-1/2}$-decay of the regularization parameter}
\label{sec:sqrtn}
We first consider regular consistency through the following proposition (proved
in Appendix~\ref{app:asympt1dbis}), then rank consistency (proposition proved in 
Appendix~\ref{app:rank1}):

\begin{proposition}
\label{prop:asympt1dbis}
Assume \hypref{model}, \hypref{samplingM} and \hypref{normalityM}. Let $\hat{W}$ be a global minimizer of \eq{problemM}.
 If $n^{1/2} \lambda_n$ tends to a limit $\lambda_0 > 0$, then $n^{1/2}(\hat{W} - \W)$ converges in distribution to the unique global minimizer of
$$
\min_{\D \in \rb^{p \times q} } \frac{1}{2}
\vect (\D)^\top \S_{mm} \vect(\D) 
-  \tr \D^\top A + \lambda_0  \left[\tr \U^\top \D \V + \| \U_\bot^\top \D \V_\bot \|_\ast \right],
$$
where $\vect (A) \in \rb^{pq}$ is normally distributed with mean zero and
covariance matrix $\sig \S_{mm}$.
\end{proposition}

\begin{proposition}
\label{prop:rank1}
Assume  \hypref{model}, \hypref{samplingM} and \hypref{normalityM}. 
If $n^{1/2} \lambda_n$ tends to a limit $\lambda_0 > 0$, then the probability that the rank of $\hat{W}$
is different from the rank of $\W$ is 
converging to $\P( \| \Lambda - \lambda_0^{-1} \Theta \|_2 \leqslant 1 ) \in (0,1)$ where 
$  \Lambda \in \rb^{( p - \r) \times (q-\r)}$ is defined in \eq{lambda} (\mysec{slow}) and
 $\Theta \in \rb^{( p - \r) \times (q-\r)} $ has a normal distribution with mean zero and covariance matrix
 $$\sigma^2 
 \left(  (\V_\bot \otimes \U_\bot)^\top \S_{mm}^{-1} (\V_\bot \otimes \U_\bot) \right)^{-1}. $$
\end{proposition}
The previous proposition ensures that the estimate $\hat{W}$ cannot be rank consistent with this decay of the regularization parameter. Note that when we take $\lambda_0$ small (i.e., we get closer to fast decays), the probability $\P( \| \Lambda - \lambda_0^{-1} \Theta \|_2 \leqslant 1 )$ tends to zero, while when we take $\lambda_0$ large (i.e., we get closer to slow decays), the same probability tends to zero or one depending on the sign of
$\|\Lambda\|_2-1$. This heuristic argument is made more precise in the following section.

\subsection{Slow decay of regularization parameter}
\label{sec:slow}
When $\lambda_n$ tends to zero more slowly than $n^{-1/2}$, the first order expansion is deterministic, as the following proposition shows (proof in Appendix~\ref{app:asympt2}):

\begin{proposition}
\label{prop:asympt2}
Assume \hypref{model}, \hypref{samplingM} and \hypref{normalityM}.  
Let $\hat{W}$ be a global minimizer of \eq{problemM}. If $n^{1/2} \lambda_n$ tends 
to $+ \infty$ and $\lambda_n$ tends to zero, then ${\lambda_n^{-1}}(\hat{W} - \W)$ converges in probability to the unique global minimizer $\Delta$ of
\BEQ
\label{eq:expansion}
\min_{\D \in \rb^{p \times q} } \frac{1}{2}
\vect (\D)^\top \S _{mm} \vect(\D) 
 +    \tr \U^\top \D \V + \| \U_\bot^\top \D \V_\bot \|_\ast  .
\EEQ
Moreover, we have $\hat{W} = \W + \lambda_n \Delta + O_p(\lambda_n + \zeta_n + \lambda_n^{-1} n^{-1/2})$.
\end{proposition}

The last proposition gives a first order expansion of $\hat{W}$ around $\W$. From Proposition~\ref{prop:rankDL} (Appendix~\ref{app:tracenorm}), we obtain immediately that if $\U_\bot^\top \Delta \V_\bot$ is different from zero, then the rank of $\hat{W}$ is ultimately strictly larger than $\r$. The condition  $\U_\bot^\top \Delta \V_\bot=0$ is thus necessary for rank consistency when $\lambda_n n^{1/2}$ tends to infinity while $\lambda_n$ tends to zero. The next lemma (proved in Appendix~\ref{lemma:cond}), gives a necessary and sufficient condition for $\U_\bot^\top \Delta \V_\bot=0$.

\begin{lemma}
\label{lemma:cond}
Assume $\S_{mm}$ is invertible, and $\W = \U \Diag(\s) \V^\top$ is the singular value decomposition of $\W$. Then the unique global minimizer of 
$$\vect (\D)^\top \S _{mm} \vect(\D) 
 +    \tr \U^\top \D \V + \| \U_\bot^\top \D \V_\bot \|_\ast  
$$
satisfies  $\U_\bot^\top \Delta \V_\bot=0$ if and only if
$$
\left\|  \left(  (\V_\bot \otimes \U_\bot)^\top \S_{mm}^{-1} (\V_\bot \otimes \U_\bot) \right)^{-1}
\left(  (\V_\bot \otimes \U_\bot)^\top \S_{mm}^{-1} (\V \otimes \U) \vect(\idm) \right)
\right\|_2 \leqslant 1.
$$
\end{lemma}

This leads to consider the matrix $\Lambda \in \rb^{( p - \r) \times (q-\r)}$ defined as
\BEQ 
\label{eq:lambda}
\vect(\Lambda) = \left(  (\V_\bot \otimes \U_\bot)^\top \S_{mm}^{-1} (\V_\bot \otimes \U_\bot) \right)^{-1}
\left(  (\V_\bot \otimes \U_\bot)^\top \S_{mm}^{-1} (\V \otimes \U) \vect(\idm) \right),
\EEQ
and the two weak and strict \emph{consistency conditions}:
\BEQ
\label{eq:weak}
\| \Lambda \|_2 \leqslant 1,
\EEQ
\BEQ
\label{eq:strict}
\| \Lambda \|_2 < 1.
\EEQ
 Note that if $\S_{mm}$ is proportional to identity, they are always satisfied because then $\Lambda = 0$.
We can now prove that the condition in \eq{strict} is sufficient for rank consistency when $n^{1/2} \lambda_n$ tends to infinity, while the condition
\eq{weak} is necessary for the existence of a sequence $\lambda_n$ such that the estimate is both consistent and rank consistent (which is a stronger result than restricting $\lambda_n$ to be tending to zero slower than $n^{-1/2}$). The following two theorems
are proved in Appendix~\ref{app:sufficient} and~\ref{app:necessary}:

\begin{theorem}
\label{theo:sufficient}
Assume \hypref{model}, \hypref{samplingM}, \hypref{normalityM}. Let $\hat{W}$ be a global minimizer of \eq{problemM}.  
If the condition in \eq{strict} is satisfied, and if $n^{1/2} \lambda_n$ tends 
to $+ \infty$ and $\lambda_n$ tends to zero, then the estimate $\hat{W}$ is consistent and rank-consistent.
\end{theorem}
 
\begin{theorem}
\label{theo:necessary}
Assume  \hypref{model}, \hypref{samplingM} and \hypref{normalityM}. 
 Let $\hat{W}$ be a global minimizer of \eq{problemM}. 
If the estimate $\hat{W}$ is consistent and rank-consistent, then the
condition in \eq{weak} is satisfied.
\end{theorem}
As opposed to the Lasso, where \eq{weak} is a necessary and sufficient condition for rank consistency~\citep{yuanlin}, this is not even true in general for the group Lasso~\citep{grouplasso}. Looking at the limiting case $\| \Lambda \|_2=1$ would similarly lead to additional but more complex  sufficient and necessary conditions, and is left out for future research.

Moreover, it may seem surprising that even when the sufficient condition \eq{strict}
is fulfilled, that the first order expansion of $\hat{W}$, i.e., $\hat{W}=
\W + \lambda_n \D + o_p(\lambda_n)$ is such that $\U_\bot^\top \D \V_\bot =0$, but nothing is said about $\U_\bot^\top \D \V$
and $\U^\top \D \V_\bot$, which are not equal to zero in general. This is due to the fact that the first $\r$ singular vectors $U$ and $V$ of $\W + \lambda_n \D $ are not fixed; indeed, the $\r$ first singular vectors (i.e., the implicit features) do rotate but with no contribution on $\U_\bot \V_\bot^\top$. This is to be contrasted with the adaptive version where asymptotically the first order expansion has constant singular vectors (see \mysec{adaptive}).

Finally, in this paper, we have only proved whether the probability of correct rank selection tends to zero or one. Proposition~\ref{prop:rank1} suggests that when $\lambda_n n^{1/2}$ tends to infinity  slowly, then this probability is close to 
$\P( \| \Lambda - \lambda_n^{-1} n^{1/2} \Theta \|_2 \leqslant 1 ) $, where $\Theta$ has a normal distribution with known covariance matrix, which converges to one exponentially fast when $\|\Lambda \|_2<1$. We are currently investigating additional assumptions under which such results are  true and thus estimate the convergence rates of the probability
of good rank selection as done by~\citet{Zhaoyu} for the Lasso.

\subsection{Factored second order moment}
Note that in the situation where $n_x$ points in $\rb^{p}$ and $n_y$ points in $\rb^{q}$ are sampled
i.i.d and a random subset of $n$ points in selected, then, we can refine the condition as follows (because
$\S_{mm} = \S_{yy} \otimes \S_{xx}$):
$$ \Lambda = ( \U_\bot^\top \S_{xx}^{-1} \U_\bot)^{-1} \U_\bot^\top \S_{xx}^{-1} \U 
\V^\top \S_{yy}^{-1} \V_\bot  ( \V_\bot^\top \S_{yy}^{-1} \V_\bot)^{-1},$$
which is equal to (by the expression of inverses of partitioned matrices):
$$\Lambda =
 ( \U_\bot^\top \S_{xx} \U)  ( \U ^\top\S_{xx}\U)^{-1}
  ( \V^\top \S_{yy}\V)^{-1}  ( \V^\top \S_{yy} \V_\bot)  .
$$
This also happens when $M_i = x_i y_i^\top$ and $x_i$ and $y_i$ independent for all $i$.

\subsection{Corollaries for the Lasso and group Lasso}
\label{sec:lasso}
For the Lasso or the group Lasso, all proposed
 results in \mysec{slow} should hold with the additional conditions that ${W}$ and $\D$ are diagonal
(block-diagonal for the group Lasso). In this situation, the singular values of the diagonal matrix $W
= \Diag(w)$ are
the norms of the diagonal blocks, while the left singular vectors are equal to the normalized versions
of the block (the signs for the Lasso). 
However, the results developed in \mysec{slow} 
do not immediately apply since
the assumptions regarding the invertibility of the second order moment matrix is not satisfied. For those problems, all matrices $M$ that are ever considered belong to a strict subspace
of $\rb^{p \times q}$ and we need to satisfy invertibility on that subspace.

More precisely, we assume that all matrices $M$ are such that 
$\vect(M) = H x $ where $H$ is a given \emph{design matrix} in $\rb^{pq \times s}$ where $s$ is the number of implicit parameter and $x \in \rb^s$. If we replace the invertibility of $\S_{mm}$ by the invertibility of $H^\top \S_{mm} H$, then all results presented in \mysec{slow} are valid, in particular, the matrix $\Lambda$ may be written as 
\begin{multline}
\label{eq:LambdaH}
\vect(\Lambda) = \left(  (\V_\bot \otimes \U_\bot)^\top H (  H^\top \S_{mm} H)
^{-1} H^\top (\V_\bot \otimes \U_\bot) \right)^\dagger \\
\times
\left(  (\V_\bot \otimes \U_\bot)^\top H ( H^\top \S_{mm}H)^{-1} H^\top (\V \otimes \U) \vect(\idm) \right),
\end{multline}
where $A^\dagger$ denotes the pseudo-inverse of $A$~\citep{golub83matrix}.

We now apply \eq{LambdaH} to the case of the group Lasso (which includes the Lasso as a special case). In this situation, we have   $M = \Diag(x_1,
\dots,x_m)$ and each $x_j \in \rb^{d_j}$, $j=1,\dots,x_m$; we consider $w$ as being defined by blocks $w_1,
\dots,w_m$, where each $w_j \in \rb^{d_j}$.
The design matrix $H$ is such that 
$ H w  = \vect( \Diag(w))$ and the matrix $H^\top \S_{mm} H$ is exactly equal to the joint
covariance matrix $\S_{xx}$ of $x=(x_1,\dots,x_m)$. Without loss of generality, we assume that the generating sparsity patttern corresponds to the first $\r$ blocks. We can then compute the singular value decomposition in closed form as
 $\U = {( \Diag(\w_i /\|\w_i\|)_{i\leqslant \r} \choose 0 } $,
$\V = {\idm \choose 0 } $ and $\s = ( \|\w_j\| )_{j \leqslant \r}$. If we let denote, for each $j$, $\mathbf{O}_j$ a basis of the subspace orthogonal to $\w_j$, we have:
$\U_\bot =\left( \begin{array}{cc}
\Diag(\mathbf{O}_i)_{i\leqslant \r} & 0 \\ 0 & \idm \end{array}\right) $ and
$\V_\bot = {0  \choose \idm } $. We can put these singular vectors into \eq{LambdaH} and get
$
 ( H^\top \S_{mm}H)^{-1} H^\top (\V \otimes \U) \vect(\idm) =
 (\S_{xx}^{-1})_{  \J, \J^c} \eta_\J
 $,
 where $\J = \{1,\dots,\r\}$ and $\eta_\J$ is the vector of normalised $\w_j$,
 $j \in \J$.
 Thus, for the group Lasso, we finally obtain:
\BEAS
   \| \Lambda \|_2 & =  & \left\| \Diag\left[ (  (\S_{xx}^{-1})_{\J^c \J^c} )^{-1} 
(\S_{xx}^{-1})_{  \J, \J^c} \eta_\J \right] \right\|_2 \\
& = &  \left\| \Diag\left[ (  \S_{xx})_{\J^c \J} 
(\S_{xx})_{  \J, \J}^{-1} \eta_\J \right] \right\|_2 
\mbox{ by the partitioned matrices inversion lemma,} \\
& = & 
\max_{ i \in \J^c }   \left\|  \S_{x_i x_\J  } \S_{x_\J  x_\J }^{-1}
\eta_\J   \right\| .
\EEAS

The condition on the invertibility of $H^\top \S_{mm} H$ is exactly the invertibility of the full joint covariance matrix of $x=(x_1,
\dots,x_m)$ and is a standard assumption for the Lasso or the group Lasso~\citep{yuanlin,Zhaoyu,zou,grouplasso}. Moreover the condition $\| \Lambda \|_2 \leqslant 1$ is exactly the one for the
group Lasso~\citep{grouplasso}, where
the pattern consistency is replaced by the consistency for the number of non zero groups.
  
 Note that we only obtain a result in terms of numbers of selected groups of variables and not in terms of the identities of the groups themselves. However, because of regular consistency, we know that at least the $\r$ true groups will be selected, and then correct model size is asymptotically equivalent to the correct groups being selected.

\section{Adaptive version}
\label{sec:adaptive}
We can follow the adaptive version of the Lasso to provide a consistent algorithm with no consistency conditions
such as \eq{weak} or \eq{strict}. More precisely,
we consider the least-square estimate $ \vect(\hat{W}_{LS}) = \hS_{mm}^{-1} \vect(\hS_{Mz})$. We have the following well known result for least-square regression:
\begin{lemma}
\label{lemma:LS}
Assume \hypref{model}, \hypref{samplingM} and \hypref{normalityM}. Then
$n^{1/2}( \hS_{mm}^{-1} \vect( \hS_{Mz}) - \vect(\W ))$ is converging in distribution to a normal distribution with zero mean
and covariance matrix $\sigma^2 \S_{mm}^{-1}$.
\end{lemma}

We consider the singular value decomposition of 
$\hat{W}_{LS} = U_{LS} \Diag(s_{LS}) V_{LS}^\top$, where $s_{LS} \geqslant 0$. With probability tending to one, $\min \{p,q\}$ singular values
are strictly positive (i.e. the rank of $\hat{W}_{LS}$ is full). We consider the \emph{full} decomposition where $U_{LS}$ and $V_{LS}$ are orthogonal \emph{square} matrices and the matrix $\Diag(s_{LS})$ is rectangular. We complete the singular values $s_{LS} \in \rb^{ \min\{p,q\}}$ by $n^{-1/2}$
to reach dimensions $p$ and $q$ (we keep the same notation for both dimensions for simplicity).

For $\gamma \in (0,1]$, we let denote 
$$A = 
U_{LS} \Diag(s_{LS})^{-\gamma} U_{LS}^\top 
\in \rb^{p \times p} \mbox{ and } B = V_{LS} \Diag(s_{LS})^{-\gamma} V_{LS}^\top
 \in \rb^{q \times q},$$ two positive definite symmetric matrices,
 and, following the adaptive Lasso of~\citet{zou},  we consider replacing
  $\| W\|_\ast$ by 
$ \| A W B \|_\ast$---note that in the Lasso special case, this exactly corresponds to the adaptive Lasso of~\citet{zou}. We obtain the following consistency theorem (proved in Appendix~\ref{app:adaptive}):
\begin{theorem}
\label{theo:adaptive}
Assume \hypref{model}, \hypref{samplingM} and \hypref{normalityM}. 
If $ \gamma \in(0,1]$, $n^{1/2} \lambda_n$ tends to 0 and $\lambda_n n^{1/2 + \gamma/2}$ tends to infinity, then any global minimizer
$\hat{W}_A$
of  
$$\frac{1}{2n} \sum_{i=1}^n ( z_i - \tr W^\top M_i)^2 + \lambda_n \| A W B \|_\ast
$$
is consistent and rank consistent. Moreover, $n^{1/2}\vect( \hat{W}_A - \W)$ is converging in distribution to a normal distribution with mean zero and covariance matrix
$$ \sigma^2 (\V \otimes \U) \left[(\V \otimes \U)^\top \S_{mm} (\V \otimes \U) \right]^{-1} (\V \otimes \U)^\top.
$$
\end{theorem}
Note the restriction $\gamma \leqslant 1$ which is due to the fact that the least-square estimate $\hat{W}_{LS}$ only estimates the singular subspaces at rate $O_p(n^{-1/2})$. In \mysec{simulations}, we illustrate the previous theorem on synthetic examples. In particular, we exhibit some singular behavior for the limiting case $\gamma=1$.

\section{Algorithms and simulations}
\label{sec:algorithms}
In this section we provide a simple algorithm to solve problems of the form
\BEQ
\label{eq:opt}
 \min_{W \in \rb^{p \times q}} \frac{1}{2}\vect(W)^\top \Sigma \vect(W) - \tr W^\top Q + \lambda \| W\|_\ast,
 \EEQ
where $\Sigma \in \rb^{pq \times pq}$ is a  positive definite matrix (note that we do not restrict $\Sigma$ to
be of the form $\Sigma = A \otimes B$ where $A$ and $B$ are  positive semidefinite matrices of size
$p \times p$ and $q \times q$). We assume that $\vect(Q)$ is in the column space of $\S$, so that the optimization problem is bounded from below (and thus the dual is feasible). In our setting, we have $\S = \hS_{mm}$ and $Q = \hS_{Mz}$.

We focus on problems where $p$ and $q$ are not too large so that we can apply Newton's method to obtain convergence up to machine precision, which is required for the fine analysis of rank consistency in \mysec{simulations}. For more efficient algorithms with larger $p$ and $q$, see~\citet{Srebro2005Maximum,Srebro2005Fast} and~\citet{lowrank}.

Because the dual norm of the trace norm is the spectral norm (see Appendix~\ref{app:tracenorm}), the dual is easily obtained as
\BEQ
\label{eq:dual}
\max_{V \in \rb^{p \times q}, \| V \|_2 \leqslant 1} - \frac{1}{2} \vect(Q - \lambda V )^\top \Sigma^{-1} \vect(Q - \lambda V ) .
\EEQ
Indeed, we have:
\BEAS
& &  \min_{W \in \rb^{p \times q}} \frac{1}{2}\vect(W)^\top \Sigma \vect(W) - \tr W^\top Q + \lambda \| W\|_\ast\\
& = & 
\min_{W \in \rb^{p \times q}} \max_{V \in \rb^{p \times q}, \| V \|_2 \leqslant 1}  \frac{1}{2}\vect(W)^\top \Sigma \vect(W) - \tr W^\top Q  + \lambda \tr V^\top W \\
& = & \max_{V \in \rb^{p \times q}, \| V \|_2 \leqslant 1}  \min_{W \in \rb^{p \times q}}  \frac{1}{2}\vect(W)^\top \Sigma \vect(W) - \tr W^\top Q  + \lambda \tr V^\top W \\
& = & \max_{V \in \rb^{p \times q}, \| V \|_2 \leqslant 1} - \frac{1}{2} \vect(Q - \lambda V )^\top \Sigma^{-1} \vect(Q - \lambda V )  , 
\EEAS
where strong duality holds because both the primal and dual problems are convex and strictly feasible~\citep{boyd}.

\subsection{Smoothing}
The problem in \eq{opt} is convex but non differentiable; in this paper we consider adding
a strictly convex function to its dual in \eq{dual} in order to make it differentiable, while controlling the increase of duality gap yielded by the added function~\citep{bonnans}.

We thus consider the following smoothing of the trace norm, namely we define
$$ F_{\varepsilon}(W) = \max_{V \in \rb^{p \times q}, \| V \|_2 \leqslant 1} \tr V^\top W -  \varepsilon  B(V),$$
where $B(V)$ is a spectral function (i.e., that depends only on singular values of $V$, equal to
$B(V) = \sum_{i=1}^{ \min\{p,q\}} b(s_i(V))$ where $b(s) = ( 1+ s) \log(1+s)
+(1-s) \log(1-s)$ if $|s| \leqslant 1$ and $+\infty$ otherwise ($s_i(V)$ denotes the $i$-th largest singular
values of $V$).  This function $F_\varepsilon$ may be computed in closed form as:
$$ F_\varepsilon(W)
= \sum_{i=1}^{ \min\{p,q\}} b^\ast(s_i(W)),
$$
where $b^\ast(s) = \varepsilon \log ( 1 + e^{v/\varepsilon}) +
 \varepsilon \log ( 1 + e^{-v/\varepsilon})  - 2 \varepsilon \log 2$. These functions
 are plotted in \myfig{barrier}; note that $| b^\ast(s) - |s||$ is uniformly bounded by $2 \log 2$.
 
 \begin{figure}
\begin{center}
\includegraphics[scale=.55]{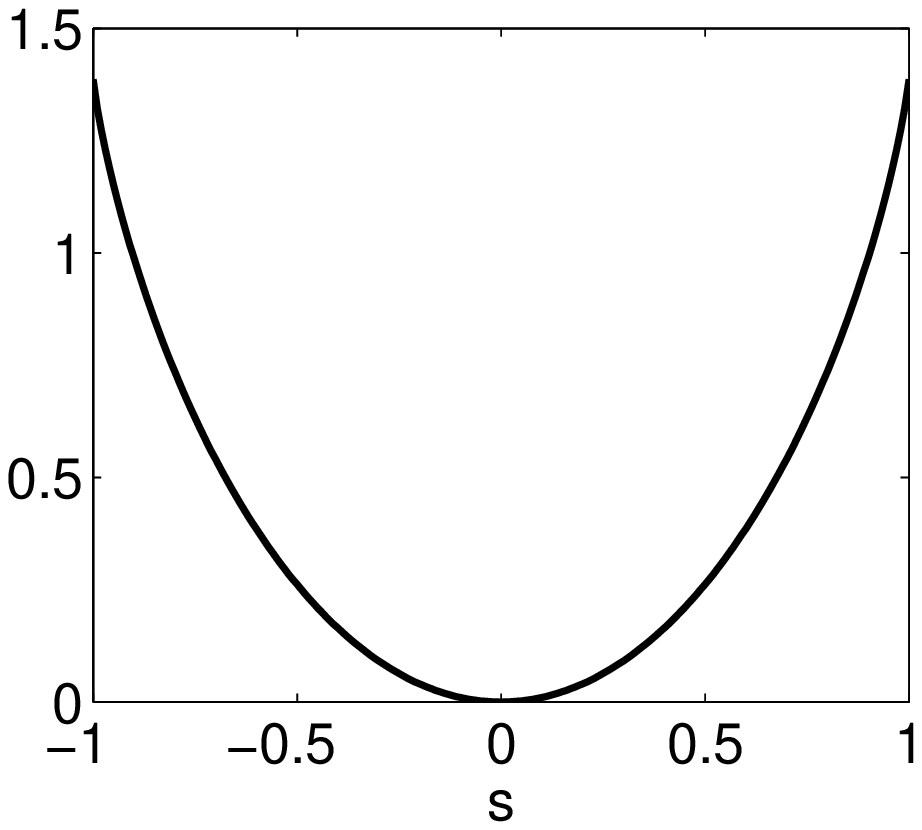} \hspace*{1cm}
\includegraphics[scale=.55]{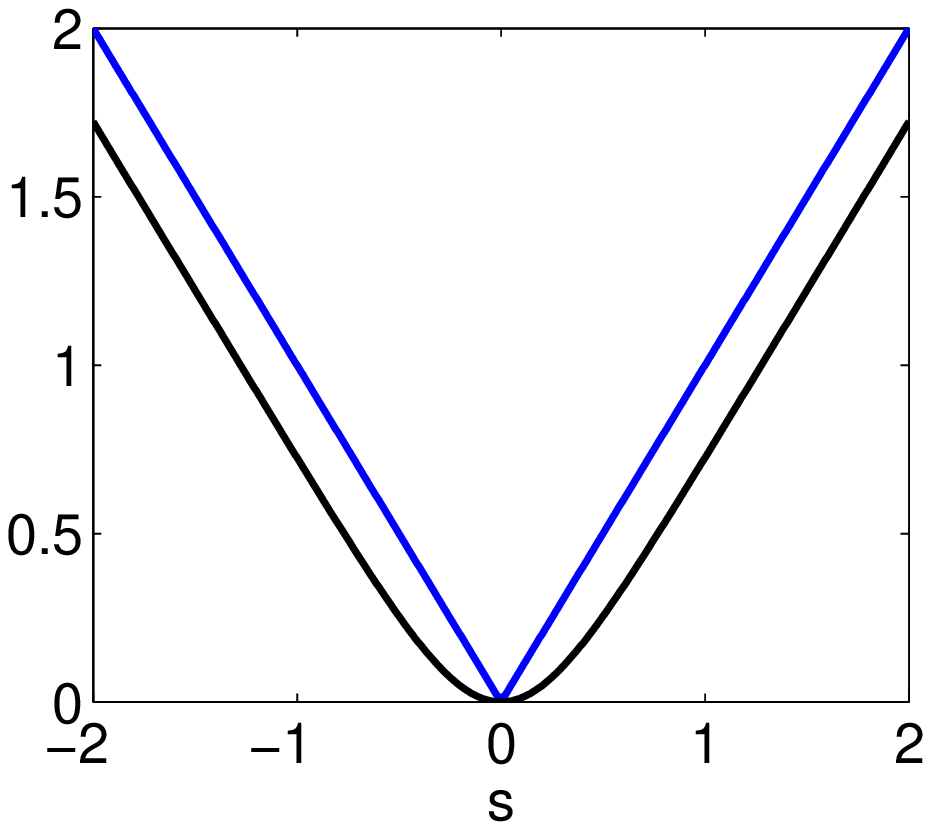}
\end{center}

\vspace*{-.5cm}

\caption{Spectral barrier functions: (left) primal function $b(s)$ and (right) dual functions $b^\ast(s)$.}
\label{fig:barrier}
\end{figure}

We finally get the following pairs of primal/dual optimization problems:
$$ \min_{W \in \rb^{p \times q}} \frac{1}{2}\vect(W)^\top \Sigma \vect(W) - \tr W^\top Q + \lambda  F_{\varepsilon/\lambda}(W), $$
$$ \max_{V \in \rb^{p \times q}, \| V \|_2 \leqslant 1} - \frac{1}{2} \vect(Q - \lambda V )^\top \Sigma^{-1} \vect(Q - \lambda V )  
 - \varepsilon B(V).$$ 
 
  We can now optimize directly in the primal formulation which is infinitely differentiable, using  Newton's method. Note that the stopping criterion
 should be an $\varepsilon \times \min \{p,q\}$ duality gap, as the controlled smoothing also leads to a small
 additional gap on the solution of the original non smoothed problem. More precisely, a duality gap of $\varepsilon \times \min \{p,q\}$ on the smoothed problem, leads to a gap of at most $(1+2\log 2) \varepsilon \times \min \{p,q\}$ for the original problem.
   
   \subsection{implementation details}
   \paragraph{Derivatives of spectral functions}
  Note that derivatives of spectral functions of the form
  $B(W) = \sum_{i=1}^{\min\{p,q\}} b(s_i(W))$, where $b$ is an even
  twice differentiable function such that $b(0)=b'(0)=0$, are easily calculated as follows; Let 
  $U \Diag(s) V^\top$ be the singular value decomposition of $W$. We then have
  the following Taylor expansion~\citep{lewis02twice}:
  $$
  B(W+\D) = B(W) + \tr \D^\top U \Diag(b'(s_i)) V^\top
  + \frac{1}{2}\sum_{i=1}^p \sum_{j=1}^q
  \frac{ b'(s_i) - b'(s_j) }{ s_i - s_j } (u_i^\top \D v_j)^2,
  $$
  where the vector of singular values is completed by zeros, and 
  $ \frac{ b'(s_i) - b'(s_j) }{ s_i - s_j }$ is defined as $b''(s_i)$ when $s_i=s_j$.
  
  \paragraph{Choice of $\varepsilon$ and computational complexity}
 Following the common practice in barrier methods we decrease the parameter geometrically after each iteration of Newton's method~\citep{boyd}. 
  Each of these Newton iterations has complexity $O(p^3 q^3)$. Empirically, the number
  of iterations does not exceed a few hundreds for solving one problem up to machine precision\footnote{MATLAB code can be downloaded from \url{http://www.di.ens.fr/~fbach/tracenorm/}}. We are currently investigating theoretical bounds on the number of iterations through
 self concordance theory~\citep{boyd}.
 
 \paragraph{Start and end of the path}
In order to avoid to consider useless values of the regularization parameter and thus use  a well adapted grid for trying several $\lambda$'s, we can consider
a specific interval for $\lambda$. When $\lambda$ is large, the solution is exactly zero, while when $\lambda$ is small, the solution
tends to $\vect(W) = \S^{-1} \vect(Q)$.

More precisely, if $\lambda$ is larger than $\| Q \|_2$, then the solution is exactly zero (because in this situation $0$ is in the
subdifferential). On the other side, we
consider for which $\lambda$, $\S^{-1} \vect(Q)$ leads to a duality gap which is less than
$\varepsilon \vect(Q)^\top \S^{-1} \vect(Q)$, where $\varepsilon$ is small. A looser condition is to take $V=0$, and the condition
becomes $ \lambda \| \S^{-1} \vect(Q) \|_\ast \leqslant \varepsilon \vect(Q)^\top \S^{-1} \vect(Q) $. Note that this is in the correct order
(i.e. lower bound smaller than upper bound ), because
$$  \vect(Q)^\top \S^{-1} \vect(Q) = \langle \vect(Q), \S^{-1} \vect(Q) \rangle \leqslant \| \S^{-1} \vect(Q)  \|_\ast \|  \vect(Q)  \|_2.$$
 This allows to design a good interval for searching for a good value of $\lambda$ or for computing the regularization path by uniform grid sampling (in log scale), or numerical path following with predictor-corrector methods such as used by~\citet{bach_thibaux}.

\subsection{Simulations}
\label{sec:simulations}
In this section, we perform simulations on toy examples to illustrate our consistency results.
We generate random i.i.d. data $\tilde{X}$ and $\tilde{Y}$ with Gaussian distributions and we select
a low rank matrix $\W$ at random and generate $Z = \diag(\tilde{X}^\top \W \tilde{Y}) + \varepsilon$ where
$\varepsilon $ has i.i.d components with normal distributions with zero mean and known variance. In this section, we always use $\r=2$,  $p=q=4$, while we consider several numbers of samples $n$, and several distributions for which the consistency conditions \eq{weak} and \eq{strict} may or may not be satisfied\footnote{Simulations may be reproduced
with MATLAB code available from \url{http://www.di.ens.fr/~fbach/tracenorm/}}. 

\begin{figure}
\begin{center}

\vspace*{-.5cm}

\hspace*{-.5cm}
\includegraphics[scale=.5]{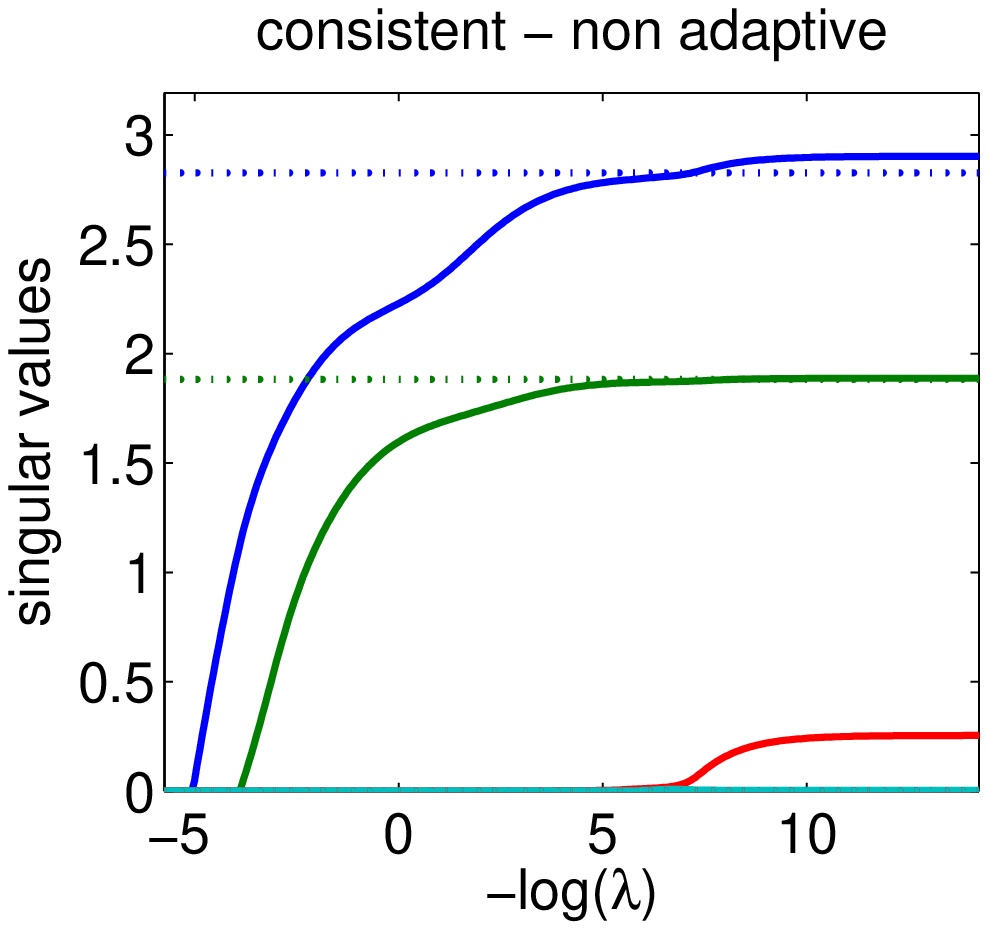} \hspace*{-.35cm}
\includegraphics[scale=.5]{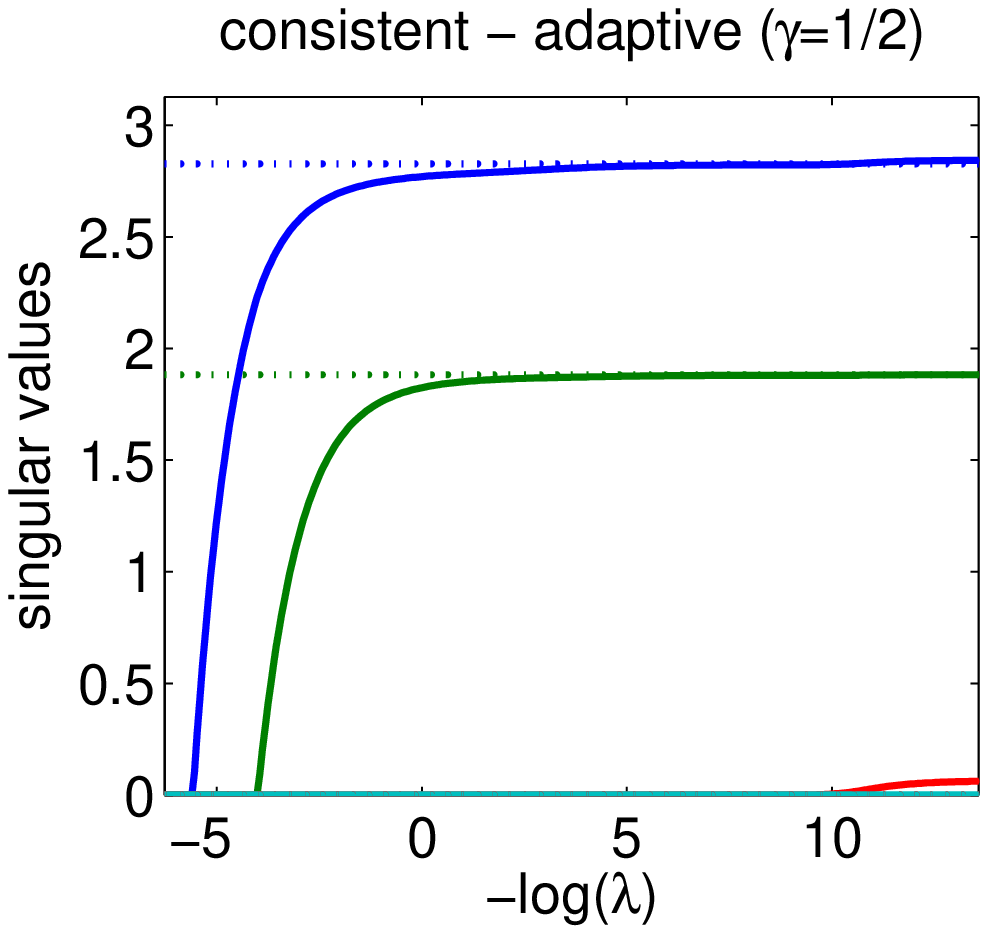}  \hspace*{-.35cm}
\includegraphics[scale=.5]{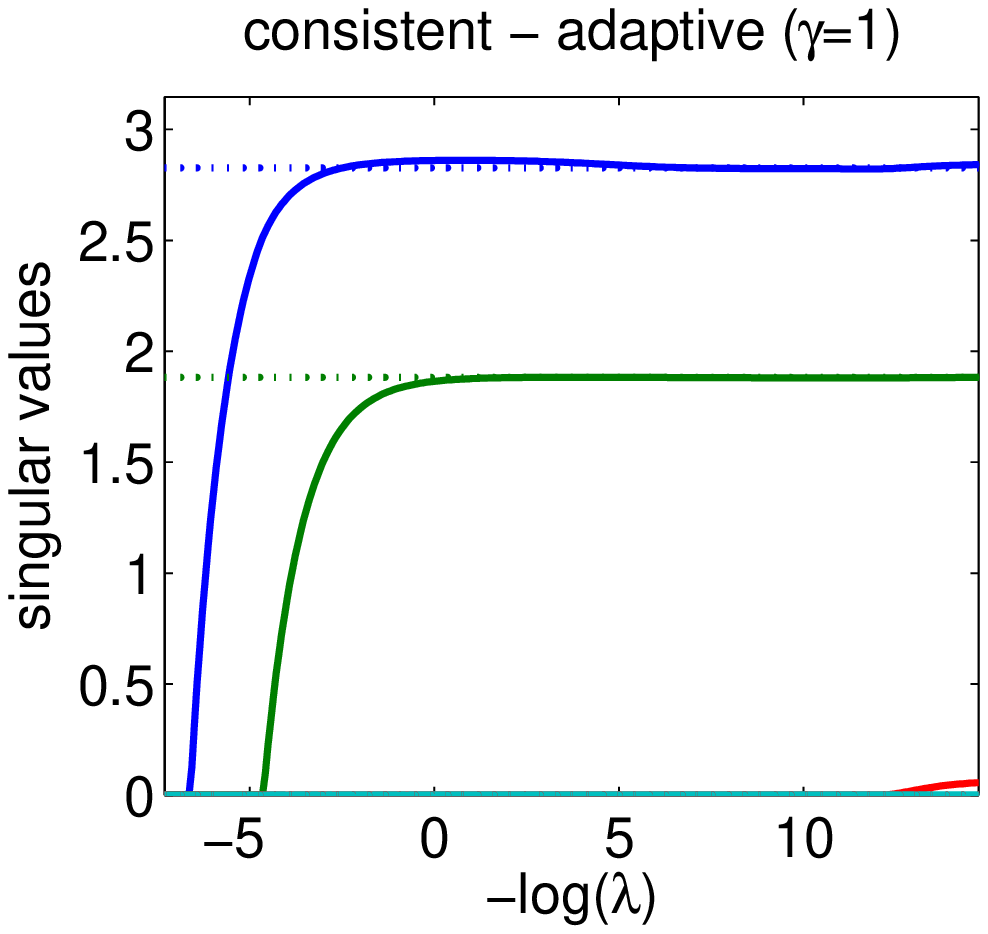}
\hspace*{-.5cm}

\vspace*{0cm}

\hspace*{-.5cm}
\includegraphics[scale=.5]{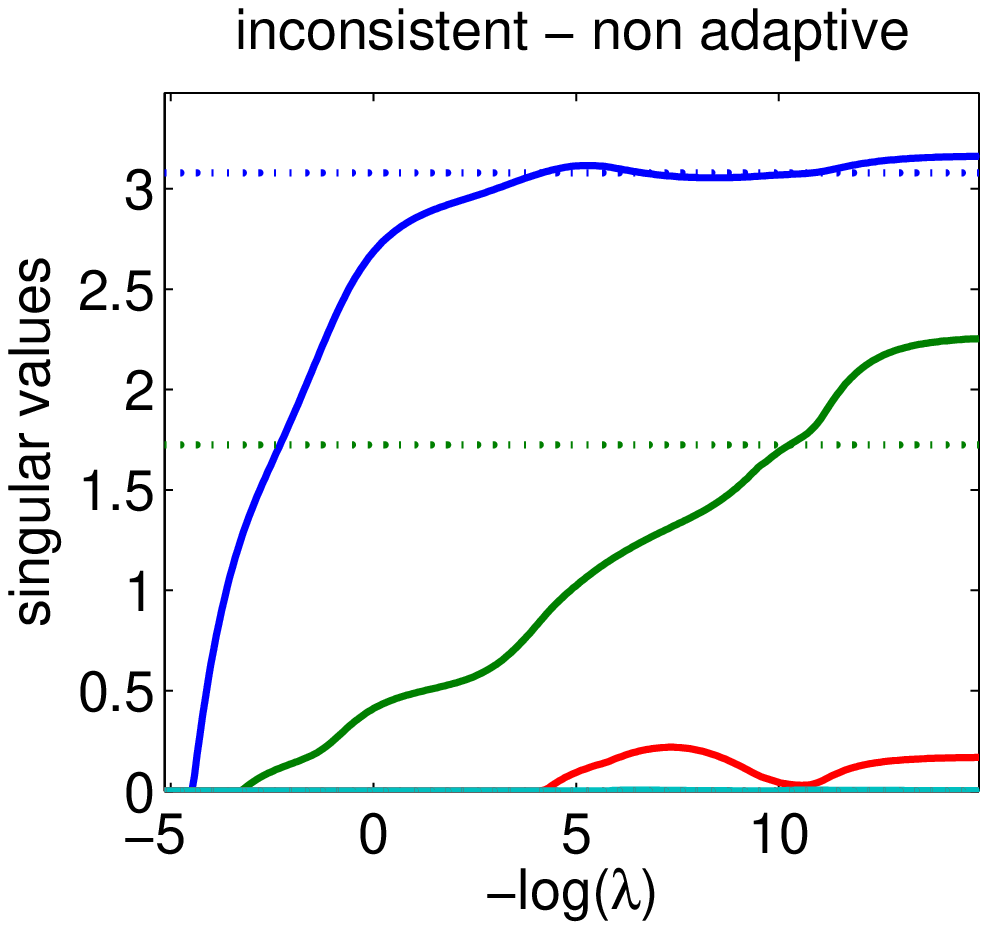} \hspace*{-.35cm}
\includegraphics[scale=.5]{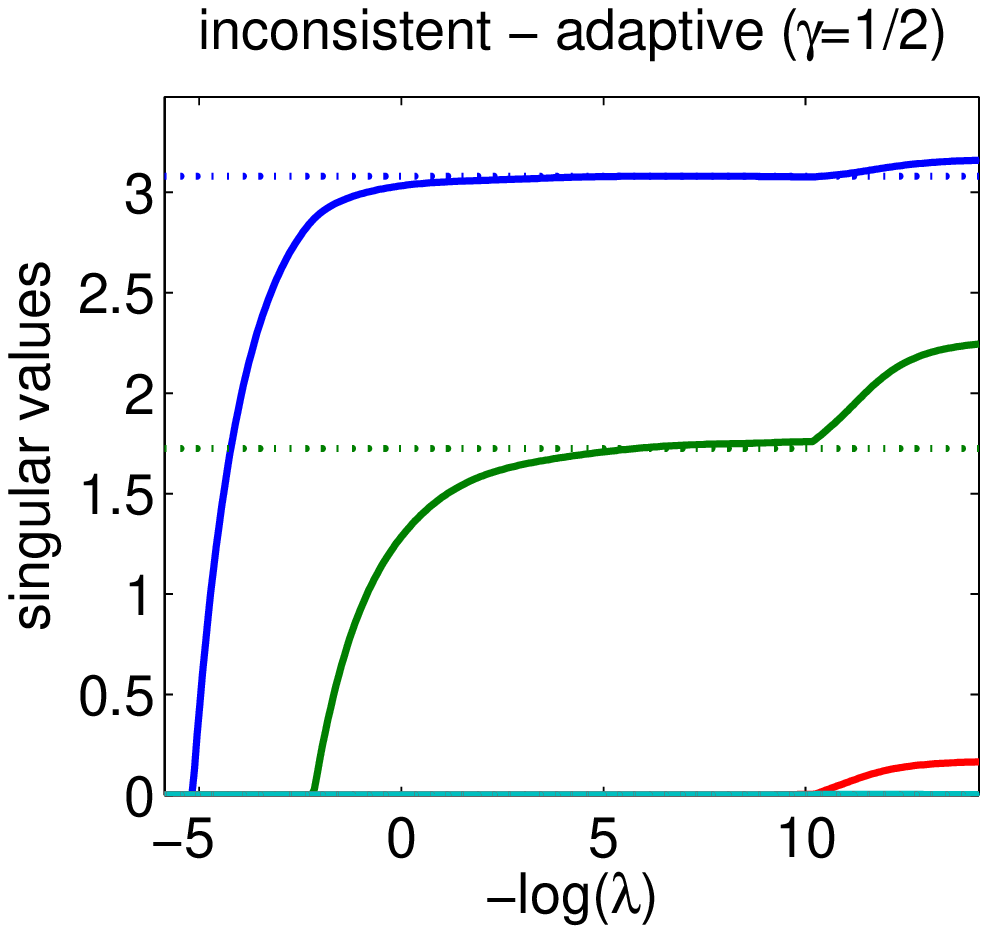}  \hspace*{-.35cm}
\includegraphics[scale=.5]{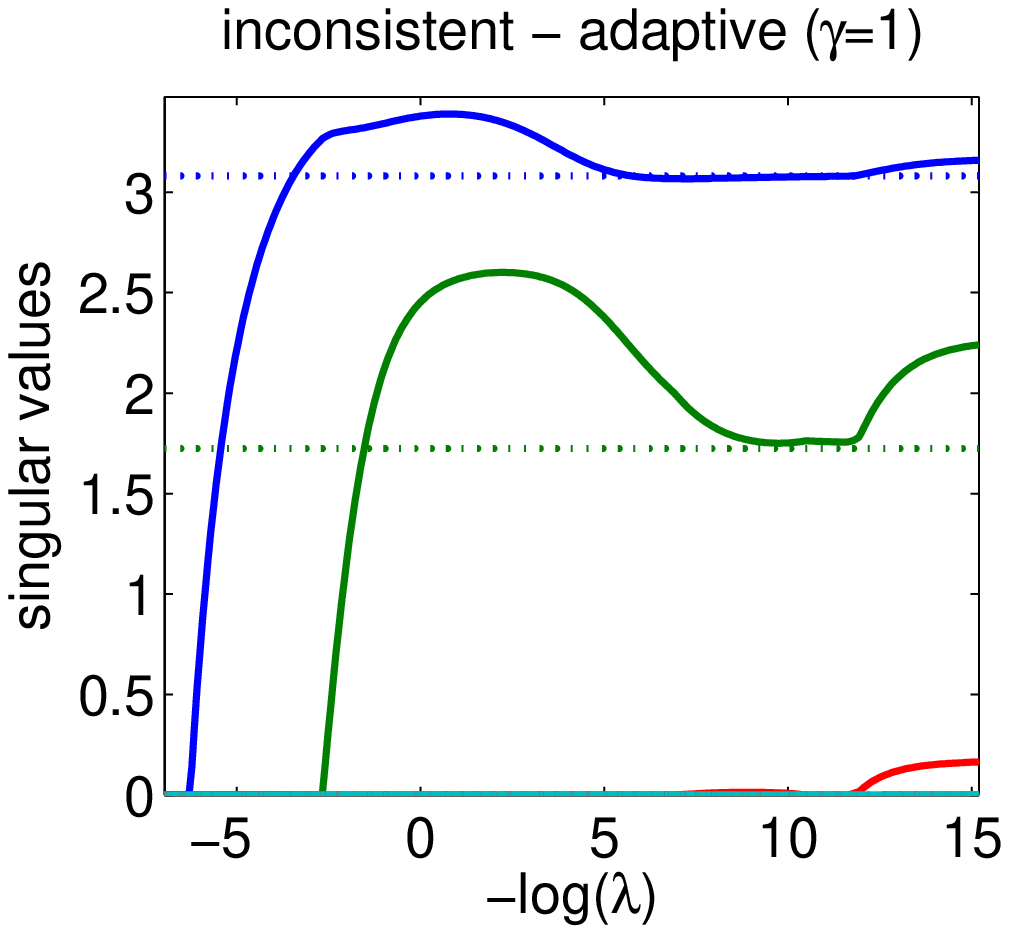}
\hspace*{-.5cm}

\vspace*{-1cm}

\end{center}
\caption{Examples of paths of singular values for $\|\Lambda\|_2 = 0.49 < 1$ (consistent, top) and  $\|\Lambda\|_2= 4.78 > 1$  (inconsistent, bottom) rank selection: regular trace norm penalization (left) and adaptive penalization with $\gamma=1/2$ (center) and $\gamma=1$ (right). Estimated singular values are plotted in plain, while population singular values are dotted.}
\label{fig:paths}
\end{figure}

In \myfig{paths}, we plot regularization paths for $n=10^3$, by showing the singular values of $\hat{W}$ compared to the singular values of $\W$, in two particular situations (\eq{weak} and \eq{strict} satisfied and not satisfied), for the regular trace norm regularization and the adaptive versions, with $\gamma=1/2$ and $\gamma=1$. Note that in the consistent case (top),  the singular values and their cardinalities are well jointly estimated, both for the non adaptive version (as predicted by Theorem~\ref{theo:sufficient}) and the adaptive versions (Theorem~\ref{theo:adaptive}), while the range of correct rank selection increases compared to the adaptive versions. However in the inconsistent case, the non adaptive regularizations scheme (bottom left) cannot achieve regular consistency together with rank consistency (Theorem~\ref{theo:necessary}), while the adaptive schemes can. Note the particular behavior of the limiting case $\gamma=1$, which still achieves both consistencies but with a singular behavior for large $\lambda$.

In \myfig{ranks-cons}, we select the distribution used for the rank-consistent case of \myfig{paths}, and compute the paths from 200 replications for $n= 10^2$, $10^3$, $10^3$ and $10^5$. For each $\lambda$, we plot the proportion of estimates with correct rank on the left plots (i.e., we get an estimation of $\P(\rank(\hat{W})=\rank(\W))$, while we plot the logarithm of the average root mean squared estimation error $\| \hat{W} - \W\|$ on the right plot. For the three regularization schemes, the range of values with high probability of correct rank selection increases as $n$ increases, and, most importantly achieves good mean squared error (right plot); in particular, for the non adaptive schemes (top plots), this corroborates the results from Proposition~\ref{prop:rank1}, which states that for $\lambda_n = \lambda_0 n^{-1/2}$ the probability tends to a limit in $(0,1)$: indeed, when $n$ increases, the value $\lambda_n$ which achieves a particular limit grows as $n^{-1/2}$, and considering the log-scale for $\lambda_n$ in \myfig{ranks-cons} and the uniform sampling for $n$ in log-scale as well, the regular spacing between the decaying parts observed in \myfig{ranks-cons} is coherent with our results.

\begin{figure}

\vspace*{-.5cm}

\begin{center}
\includegraphics[scale=.5]{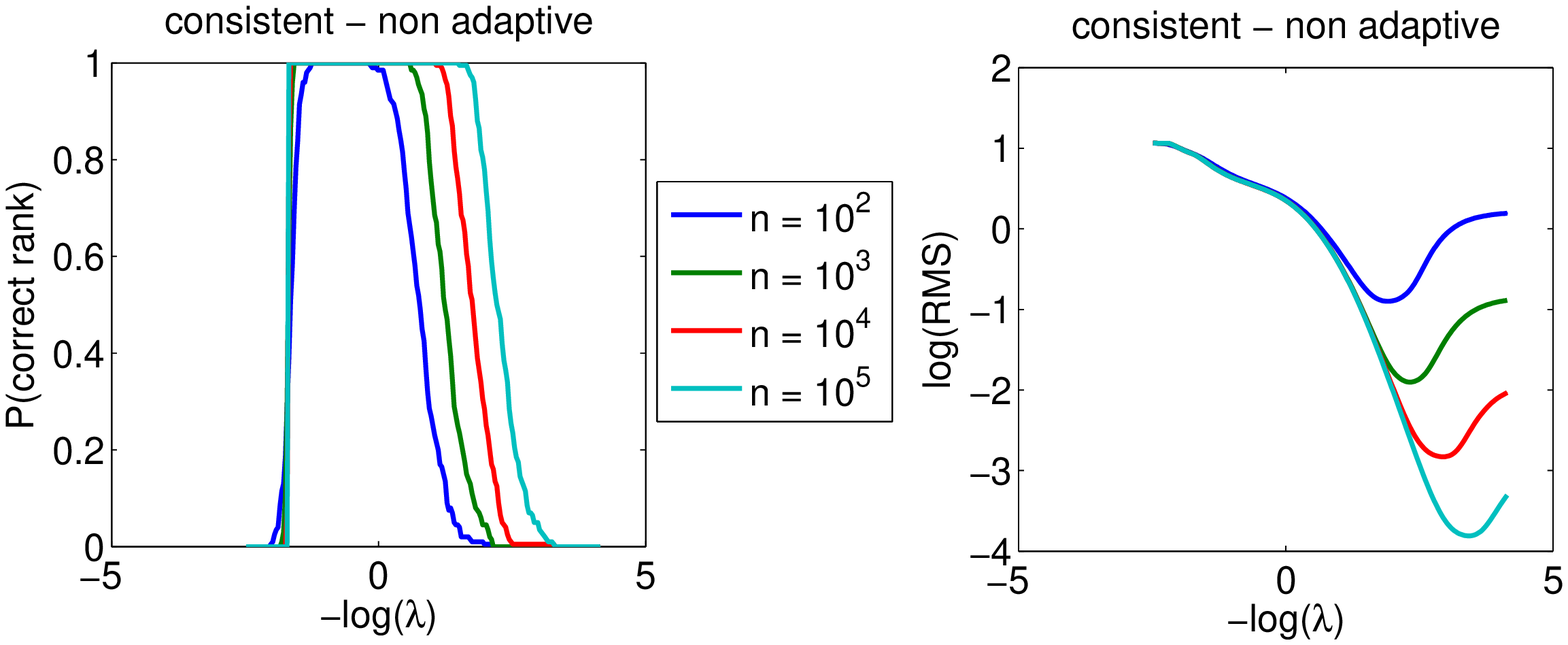}  

\vspace*{.5cm}

\includegraphics[scale=.5]{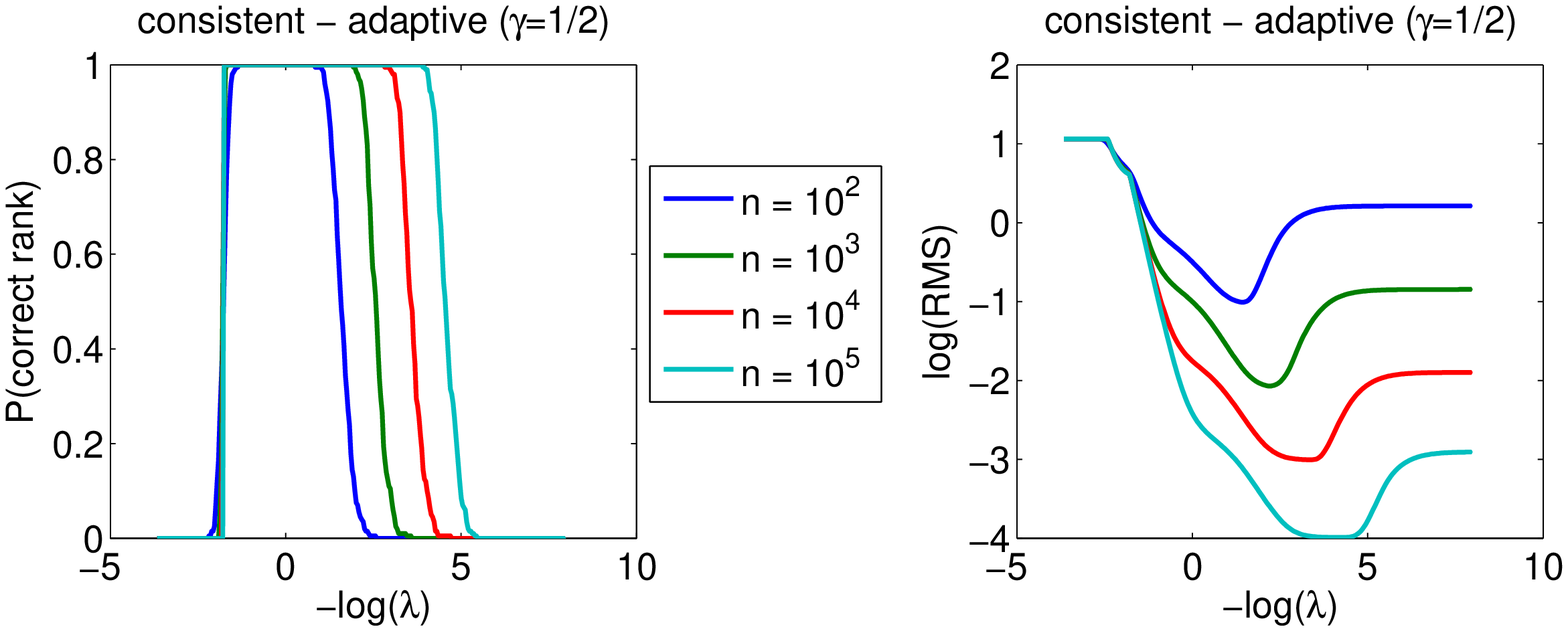}  

\vspace*{.5cm}

\includegraphics[scale=.5]{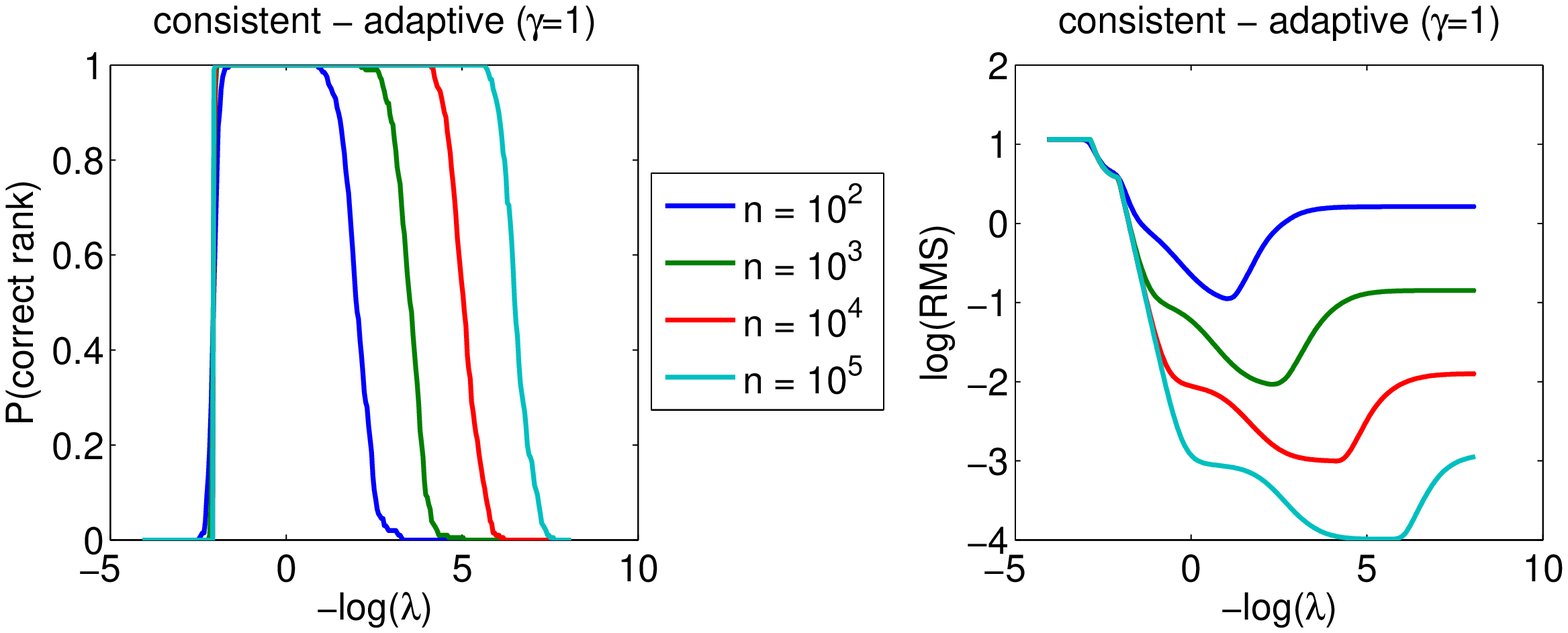}  

\end{center}

\vspace*{-.5cm}

\caption{Synthetic example where consistency condition in \eq{strict} is satisfied: probability of correct rank selection (left) and  logarithm of the expected mean squared estimation error (right), for several number of samples as a function of the regularization parameter, for regular regularization (top), adaptive regularization with $\gamma=1/2$ (center) and $\gamma=1$ (bottom). }
\label{fig:ranks-cons}
\end{figure}

In \myfig{ranks-inc}, we perform the same operations but with the inconsistent case of \myfig{paths}. For the non adaptive case (top plot), the range of values of $\lambda$ that achieve high probability of correct rank selection does not increase when $n$ increases and stays bounded, in places where the estimation error is not tending to zero: in the inconsistent case, the trace norm regularization does not manage to solve the trade-off between rank consistency and regular consistency. However, for the adaptive versions, it does, still with a somewhat singular behavior of the limiting case $\gamma=1$.

\begin{figure}

\vspace*{-.5cm}

\begin{center}
\includegraphics[scale=.5]{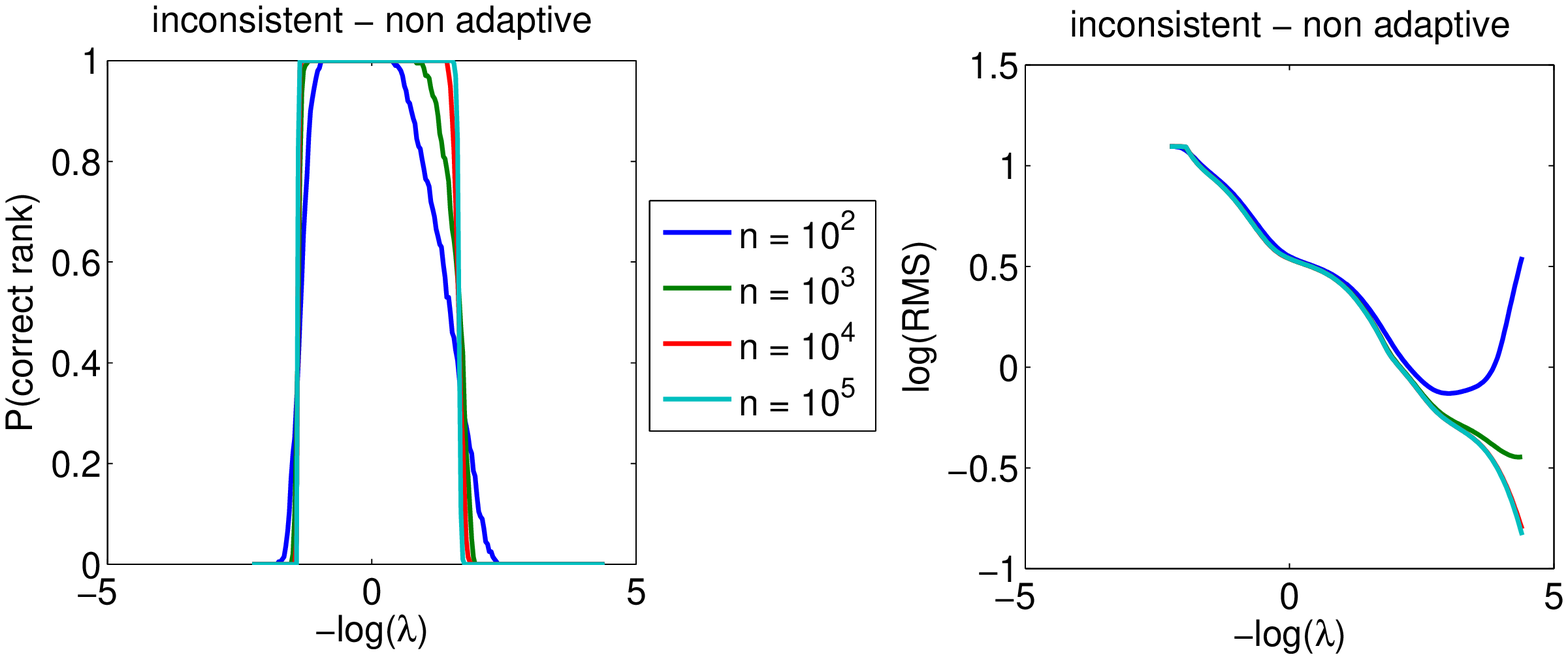}  

\vspace*{.5cm}

\includegraphics[scale=.5]{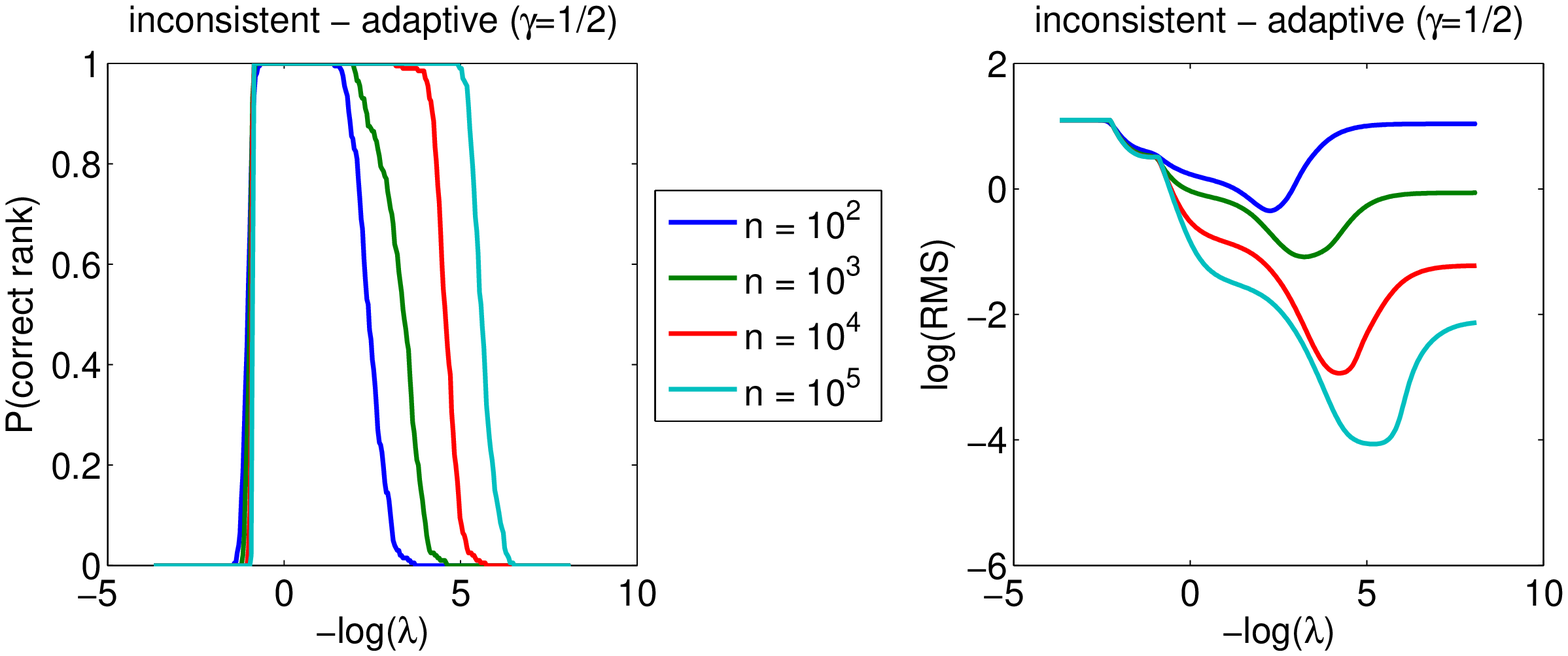}  

\vspace*{.5cm}

\includegraphics[scale=.5]{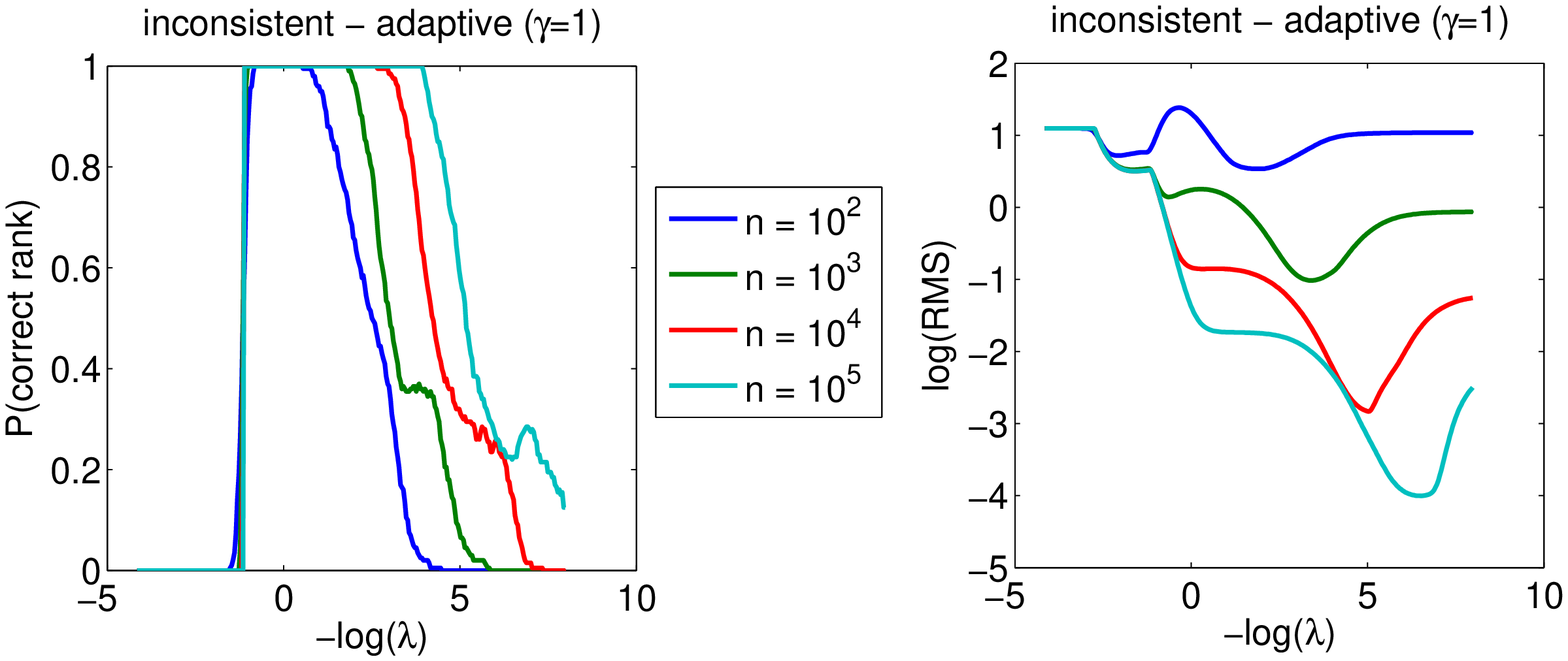}  

\end{center}

\vspace*{-.5cm}

\caption{Synthetic example where consistency condition in \eq{weak} is not satisfied: probability of correct rank selection (left) and  logarithm of the expected mean squared estimation error  (right), for several number of samples as a function of the regularization parameter, for regular regularization (top), adaptive regularization with $\gamma=1/2$ (center) and $\gamma=1$ (bottom). }
\label{fig:ranks-inc}
\end{figure}

Finally, in \myfig{lambdas}, we consider 400 different distributions with various values of
$\| \Lambda \|$ smaller or greater than one, and computed the regularization paths with $n=10^3$ samples. From the paths, we consider the estimate $\hat{W}$ with correct rank and best distance to $\W$ and plot the best error versus $\log_{10}( \| \Lambda \|_2)$. For positive values of  $\log_{10}( \| \Lambda \|_2)$, the best error is far from zero, and the error grows with the distance to zero; while for negative values, we get low errors with lower errors for small $\log_{10}( \| \Lambda \|_2)$, corroborating the influence of $\|\lambda\|_2$ described in Proposition~\ref{prop:rank1}.

\begin{figure}
\begin{center}
\includegraphics[scale=.35]{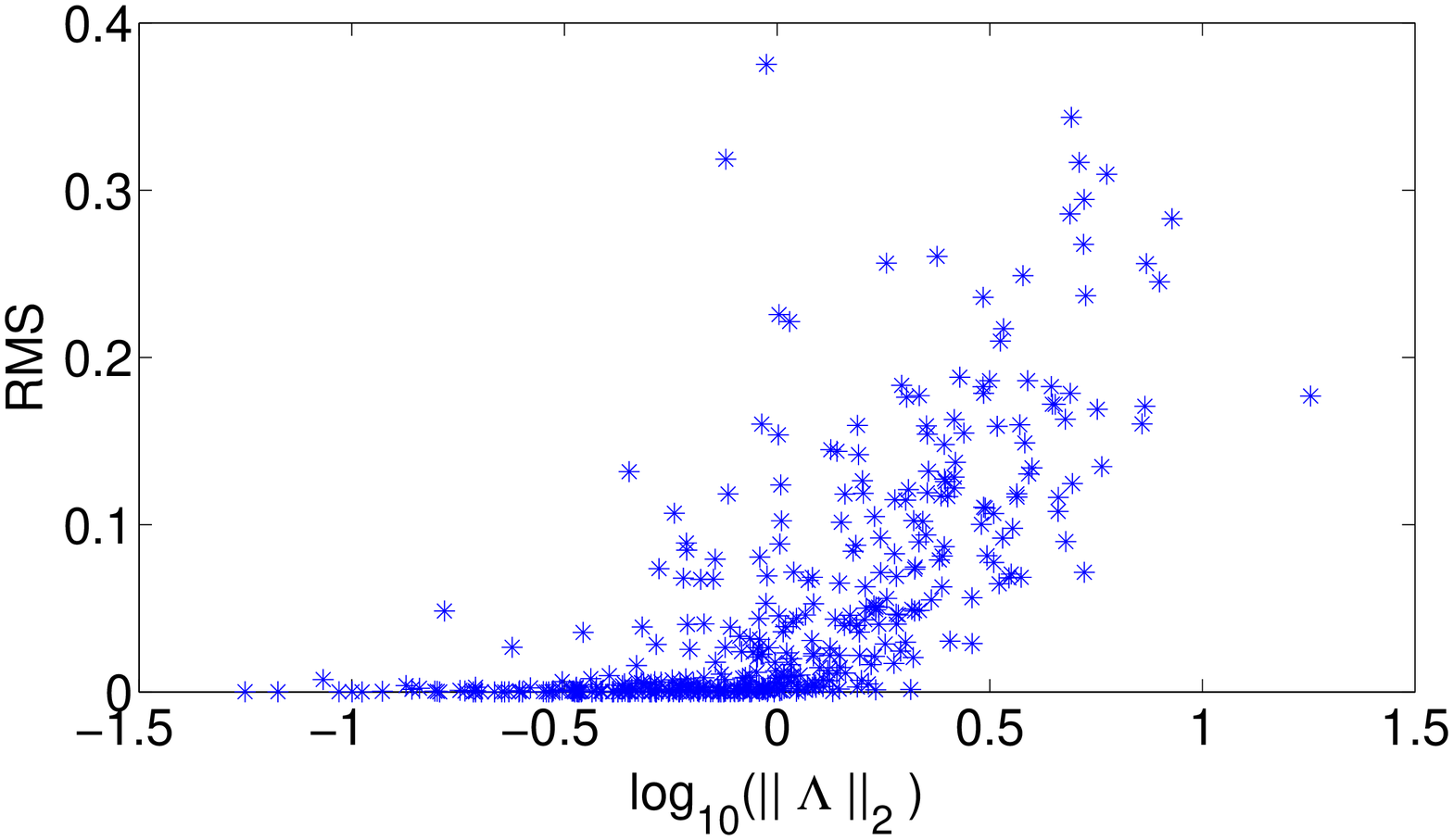}  

\vspace*{-.5cm}

\end{center}
\caption{Scatter plots of $\log_{10}( \| \Lambda \|_2)$ versus the squared error of the best estimate with correct rank (i.e., such that
$\rank(\hat{W})  =\r$ and  $\| \hat{W} - \W\|$ as small as possible). See text for details. }
\label{fig:lambdas}
\end{figure}

\section{Conclusion}
We have presented an analysis of the rank consistency for the penalization by the trace norm, and derived general necessary and sufficient conditions. This work can be extended in several interesting ways: first, by going from the square loss to more general losses, in particular for other types of supervised learning problems such as classification; or by looking at the collaborative filtering setting where  only some of the attributes are observed~\citep{lowrank} and dimensions $p$ and $q$ are allowed to grow. Moreover, we are currently pursuing non asymptotic extensions of the current work, making links with the recent work of~\citet{fazel2} and of~\citet{mein}.

\vspace*{1cm}

\appendix
\section{Tools for analysis of singular value decomposition}

\label{app:svd}
In this appendix, we review and derive precise results regarding singular value decompositions. We consider $W 
\in \rb^{ p \times q}$ and we let denote $W = U \Diag(s) V^\top$ its singular value decomposition with
$U \in \rb^{p \times r}$, $V \in \rb^{q \times r}$ with orthonormal columns, and $s \in \rb^r$ with
strictly positive values ($r$ is the rank of $W$). Note that when a singular value $s_i$ is simple, i.e., does not coalesce with any other singular values, then the vectors $u_i$ and $v_i$ are uniquely defined up to simultaneous sign flips, i.e., only the matrix $u_i v_i^\top$ is unique. However, when some singular values coalesce, then the corresponding singular vectors are
defined up to a rotation, and thus in general care must be taken and considering isolated singular vectors should be avoided~\citep{sun}. All tools presented in this appendix are robust to the particular choice of the singular vectors.

\subsection{Jordan-Wielandt matrix}
We use the fact that singular values of $W$ can be obtained from the eigenvalues of the Jordan-Wielandt  matrix
$\bar{W} = \left( \begin{array}{cc} 0 & W \\ W^\top & 0  \end{array} \right) \in \rb^{ (p+q) \times (p+q)}$~\citep{sun}. Indeed this matrix has eigenvalues $s_i$ and $ -s_i$, $i=1,\dots,r$, where
$s_i$ are the (strictly positive) singular values of $W$, with eigenvectors $ \frac{1}{\sqrt{2}} \left( \begin{array}{c} u_i \\ v_i  \end{array} \right)$
and
$ \frac{1}{\sqrt{2}} \left( \begin{array}{c} u_i \\ -v_i  \end{array} \right)$ where $u_i,v_i$ are the left and right associated
singular vectors. Also, the remaining eigenvalues are all equal to zero, with eigensubspace (of dimension
$p+q-2r$)  composed of all
$\left( \begin{array}{c} u \\v \end{array} \right)$ such that  for all $i \in \{1,\dots, r\}$, $u^\top u_i = v^\top v_i= 0$. We let denote $\bar{U}$ the eigenvectors of $\bar{W}$ corresponding to non zero eigenvalues in $\bar{S}$. We have
$\bar{U} = \frac{1}{\sqrt{2}}
\left( \begin{array}{cc}U & U \\ V & - V  \end{array} \right)$ and
$\bar{S} = \frac{1}{\sqrt{2}}
\left( \begin{array}{cc} \Diag(s) & 0 \\ 0 & - \Diag(s)  \end{array} \right)$ and
  $\bar{W} = \bar{U} \bar{S} \bar{U}^\top$, 
$ \bar{U}  \bar{U}^\top
=\left( \begin{array}{cc} UU^\top & 0 \\ 0 & VV^\top  \end{array} \right)
$, and $ \bar{U}  { \rm sign} (\bar{S}) \bar{U}^\top
=\left( \begin{array}{cc} 0 &  UV^\top \\ V U^\top & 0  \end{array} \right)
$.

\subsection{Cauchy residue formula and eigenvalues}
Given the matrix $\bar{W}$, and a simple closed curve $\mathcal{C}$ in the complex plane that does
 not go through any of the eigenvalues of $\bar{W}$, then
$$\Pi_\mathcal{C}(\bar{W}) = \frac{1}{2i\pi} \oint_\mathcal{C} \frac{ d\lambda}{\lambda \idm - \bar{W}}$$ is equal to the orthogonal projection onto the orthogonal sum
of all eigensubspaces of $\bar{W}$ associated with eigenvalues in the interior of $\mathcal{C}$~\citep{kato}. This is easily seen  by writing down the eigenvalue decomposition and the Cauchy residue formula ($\frac{1}{2i\pi} \oint_\mathcal{C} \frac{ d\lambda}{\lambda - \lambda_i} = 1$ if $\lambda_i$
is in the interior ${\rm int} ( \mathcal{C})$ of $\mathcal{C}$ and $0$ otherwise), and:
$$
\frac{1}{2i\pi} \oint_\mathcal{C} \frac{ d\lambda}{\lambda \idm - \bar{W}}
= \sum_{i=1}^{2r } \bar{u}_i \bar{u}_i^\top  \times 
\frac{1}{2i\pi} \oint_\mathcal{C} \frac{ d\lambda}{\lambda  - \bar{s}_i } =
\sum_{ i , \ \bar{s}_i \in  {\rm int} ( \mathcal{C})} u_i u_i^\top.
$$
See~\citet{rudin} for an introduction to complex analysis and Cauchy residue formula.
Moreover, we can obtain the restriction of $\bar{W}$ onto a specific eigensubspace as:
$$
\bar{W} \Pi_\mathcal{C}(\bar{W}) = 
\frac{1}{2i\pi} \oint_\mathcal{C} \frac{ \bar{W} d\lambda}{\lambda \idm - \bar{W}}
=  - 
\frac{1}{2i\pi} \oint_\mathcal{C} \frac{\lambda d\lambda}{\lambda \idm - \bar{W}}.
$$
We let denote $s_1$ and $s_r$ the largest and smallest strictly positive singular values of $W$; if $\| \D \|_2 < s_r / 2$, then $W+\Delta$ has $r$ singular values strictly greater than $s_r/2$ and the
remaining ones are strictly less than $s_r/2$~\citep{sun}. Thus, if we denote
$\mathcal{C}$  the oriented circle of radius  $s_r/2$, $\Pi_\mathcal{C}(\bar{W})$ is the projector on the $p+q-2r$-dimensional 
null space of $\bar{W}$, and
for any $\D$ such that $\| \D \|_2 < s_r/2$, 
$\Pi_\mathcal{C}( \bar{W} + \bar{\D} ) $ is also the projector on the $p+q-2r$-dimensional invariant subspace of
$\bar{W}+\bar{\D}$, which corresponds to the smallest eigenvalues. We let denote
$\Pi_o(\bar{W}+\bar{\Delta})$ that projector and $\Pi_r(\bar{W}+\bar{\Delta})
= \idm - \Pi_o(\bar{W}+\bar{\Delta}) $ the orthogonal projector (which is the projection onto the $2r$-th principal subspace).

We can now find expansions around $\Delta=0$ as follows:
\BEAS
\Pi_o(\bar{W}+\bar{\D}) - \Pi_o(\bar{W})  & = & 
 \frac{1}{2i\pi} \oint_\mathcal{C}  (\lambda \idm - \bar{W})^{-1} \bar{\D} (\lambda \idm - \bar{W} -  \bar{\D})^{-1}  d\lambda \\
 & = & 
 \frac{1}{2i\pi} \oint_\mathcal{C}  (\lambda \idm - \bar{W})^{-1} \bar{\D} (\lambda \idm - \bar{W} )^{-1}  d\lambda  \\
 &  &  + \frac{1}{2i\pi} \oint_\mathcal{C}  (\lambda \idm - \bar{W})^{-1}  \bar{\D}(\lambda \idm - \bar{W})^{-1} \bar{\D}
   (\lambda \idm - \bar{W} - \bar{\D})^{-1}  d\lambda, 
\EEAS
and
\BEAS
(\bar{W}+\bar{\D}) \Pi_o(\bar{W}+\bar{\D})\! -\! \bar{W} \Pi_o(\bar{W}) \!\!\!\! &  = & 
 - \frac{1}{2i\pi} \oint_\mathcal{C}  \lambda (\lambda \idm - \bar{W})^{-1} \bar{\D} (\lambda \idm - \bar{W} -  \bar{\D})^{-1}  d\lambda \\
 & = & 
- \frac{1}{2i\pi} \oint_\mathcal{C}  \lambda (\lambda \idm - \bar{W})^{-1} \bar{\D} (\lambda \idm - \bar{W} )^{-1}  d\lambda  \\
 &  & \!\!\!\!\! - \frac{1}{2i\pi} \oint_\mathcal{C}  \lambda (\lambda \idm - \bar{W})^{-1}  \bar{\D}(\lambda \idm - \bar{W})^{-1} \bar{\D}
   (\lambda \idm - \bar{W} - \bar{\D})^{-1}  d\lambda, 
\EEAS
which lead to the following two propositions:
\begin{proposition}
Assume $W$ has rank $r$ and $\| \D \|_2 < s_r / 4$ where $s_r$ is the smallest positive singular value of $W$. Then the projection $\Pi_r(\bar{W})$ on the first $r$ eigenvectors
of $\bar{W}$ is such that
$$
\| \Pi_o(\bar{W}+\bar{\D}) - \Pi_o(\bar{W}) \|_2
\leqslant  \frac{4}{   s_r  } \|  {\D} \|_2
$$
and
$$
\| \Pi_o(\bar{W}+\bar{\D}) - \Pi_o(\bar{W})
-(\idm - \bar{U}\bar{U}^\top)\bar{\D}\bar{U} \bar{S}^{-1} \bar{U}^\top -\bar{U}\bar{S}^{-1} \bar{U}^\top \bar{\D} (\idm - \bar{U}\bar{U}^\top)
  \|_2 
\leqslant  \frac{8}{s_r^2 } \| \bar{\D} \|_2^2.
$$
\end{proposition}
\begin{proof}
For $\lambda \in \mathcal{C}$ we have: 
$ \| ( \lambda \idm - \bar{W} )^{-1} \|_2 \geqslant 2/s_r$ and
$ \| ( \lambda \idm - \bar{W} - \bar{\Delta} )^{-1} \|_2 \geqslant 4/s_r$, which implies
\BEAS
\| \Pi_r(\bar{W}+\bar{\D}) - \Pi_r(\bar{W}) \|_2
& \leqslant
&   \frac{1}{2\pi} \oint_\mathcal{C}  \| ( \lambda \idm - \bar{W} )^{-1} \|_2
\| \Delta \|_2  \| ( \lambda \idm - \bar{W} - \bar{\Delta} )^{-1} \|_2
\\
& \leqslant &   \left( \frac{1}{2\pi} 2\pi \frac{s_r}{2} \right) \|\Delta\|_2 
\frac{2}{s_r} \frac{4}{s_r}. \EEAS
 In order to prove the other result, we simply need to compute:
 \BEAS
 \frac{1}{2i\pi} \oint_\mathcal{C}  (\lambda \idm - \bar{W})^{-1} \bar{\D} (\lambda \idm - \bar{W} )^{-1}  d\lambda  
 & = &
\sum_{i,j} \bar{u}_i \bar{u}_i^\top \Delta \bar{u}_j \bar{u}_j^\top  \frac{1}{2i\pi} \oint_\mathcal{C}  \frac{1}{(\lambda- \bar{s}_i)(\lambda- \bar{s}_j)} d\lambda   \\
& = &\!\!\!
\sum_{i,j} \bar{u}_i \bar{u}_i^\top \Delta \bar{u}_j \bar{u}_j^\top  
\left( \frac{ 1_{ i \notin {\rm int}(\mathcal{C})}  1_{j \in {\rm int}(\mathcal{C})}}{\bar{s}_i} 
+ \frac{ 1_{ j \notin {\rm int}(\mathcal{C})}  1_{i \in {\rm int}(\mathcal{C})}}{\bar{s}_j}
\right)  \\
& = &    
(\idm - \bar{U}\bar{U}^\top)\bar{\D}\bar{U} \bar{S}^{-1} \bar{U}^\top +\bar{U}\bar{S}^{-1} \bar{U}^\top \bar{\D} (\idm - \bar{U}\bar{U}^\top) . 
 \EEAS 
\end{proof}

\begin{proposition}
\label{prop:proj}
Assume $W$ has rank $r$ and $\| \D \|_2 < s_r / 4$ where $s_r$ is the smallest positive singular value of $W$. Then the projection $\Pi_r(\bar{W})$ on the first $r$ eigenvectors
of $\bar{W}$ is such that
$$
\| \Pi_o(\bar{W}+\bar{\D}) (\bar{W}+\bar{\D}) - \Pi_o(\bar{W})  \bar{W}\|_2
\leqslant  2 \|  {\D} \|_2
$$
and
$$
\|\Pi_o(\bar{W}+\bar{\D}) (\bar{W}+\bar{\D}) - \Pi_o(\bar{W})  \bar{W}
+(\idm - \bar{U}\bar{U}^\top)\bar{\D}  (\idm - \bar{U}\bar{U}^\top)
  \|_2 
\leqslant  \frac{4}{s_r } \| \bar{\D} \|_2^2.
$$
\end{proposition}
\begin{proof}
For $\lambda \in \mathcal{C}$ we have: 
$ \| ( \lambda \idm - \bar{W} )^{-1} \|_2 \geqslant 2/s_r$ and
$ \| ( \lambda \idm - \bar{W} - \bar{\Delta} )^{-1} \|_2 \geqslant 4/s_r$, which implies
\BEAS
\| \Pi_r(\bar{W}+\bar{\D}) - \Pi_r(\bar{W}) \|_2
& \leqslant
&   \frac{1}{2 \pi} \oint_\mathcal{C} |\lambda|  \| ( \lambda \idm - \bar{W} )^{-1} \|_2
\| \Delta \|_2  \| ( \lambda \idm - \bar{W} - \bar{\Delta} )^{-1} \|_2
\\
& \leqslant & \left( \frac{1}{2\pi} 2\pi \frac{s_r}{2} \right) \frac{s_r}{2} \|\Delta\|_2 
\frac{2}{s_r} \frac{4}{s_r}.
\EEAS
 In order to prove the other result, we simply need to compute:
 \BEAS
- \frac{1}{2i\pi} \oint_\mathcal{C}  \lambda (\lambda \idm - \bar{W})^{-1} \bar{\D} (\lambda \idm - \bar{W} )^{-1}  d\lambda  
 & = & -
\sum_{i,j} \bar{u}_i \bar{u}_i^\top \Delta \bar{u}_j \bar{u}_j^\top  \frac{1}{2i\pi} \oint_\mathcal{C}  \frac{\lambda }{(\lambda- \bar{s}_i)(\lambda- \bar{s}_j)} d\lambda   \\
& = &\!\!\!
- \sum_{i,j} \bar{u}_i \bar{u}_i^\top \Delta \bar{u}_j \bar{u}_j^\top  
\left( 1_{ i \in {\rm int}(\mathcal{C})}  1_{j \in {\rm int}(\mathcal{C})} \right) \\
& = & - (\idm - \bar{U}\bar{U}^\top)\bar{\D}  (\idm - \bar{U}\bar{U}^\top).
 \EEAS 
\end{proof}

The variations of $\Pi(\bar{W})$ translates immediately into variations of the singular projections $UU^\top $ and $VV^\top$. Indeed we get that the first order variation of $UU^\top$ is $- (\idm - UU^\top) \D V S^{-1} U ^\top$ and the variation of $V$ is
equal to  $ -(\idm - VV^\top) \D^\top U S^{-1} V^\top$, with errors bounded in spectral norm by 
$ \frac{8}{s_r^2 } \|  {\D} \|_2^2$.  Similarly, when restricted to the small singular values, the first order expansion is $ (\idm - UU^\top) \Delta (\idm - VV^\top)$, with error term
bounded in spectral norm by
$ \frac{4}{s_r } \|  {\D} \|_2^2$. Those results lead to the following proposition that gives a local sufficient condition for $\rank(W+\Delta)>\rank(W)$:
\begin{proposition}
\label{prop:rankDL}
Assume $W$ has rank $r<\min \{p,q\}$ with ordered
singular value decomposition $W= U \Diag(s) V^\top$. If
$ \frac{4}{s_r} \| \D \|_2^2 < \|
(\idm - UU^\top) \D (\idm - VV^\top)  \|_2$, then $\rank(W+\Delta)> r$.
\end{proposition}

\section{Some facts about the trace norm}
\label{app:tracenorm}
In this appendix, we review known properties of the trace norm that we use in this paper. Most of the results are extensions of similar results for the $\ell_1$-norm on vectors. First, we have the following result:
\begin{lemma} \emph{\textbf{\citep[Dual norm,][]{fazel}}}
The trace norm $\| \cdot \|_\ast$ is a norm and its dual norm is the operator norm $\| \cdot \|$.
\end{lemma}
Note that the dual norm $N(W)$ is defined as~\citep{boyd}:
$$N(W) = \sup_{ \|V \|_\ast \leqslant 1 } \tr W^\top V.$$
This immediately implies the following result:
\begin{lemma} \emph{\textbf{(Fenchel conjugate)}}
\label{lemma:fenchel}
We have: $ \displaystyle \max_{W \in \rb^{p \times q}} \tr W^\top V - \| W\|_\ast = 0$ if $\|V\|\leqslant 1$ and $+\infty$ otherwise.
\end{lemma}

In this paper, we need to compute the subdifferential and directional derivatives of the trace norm. We have from~\citet{fazel2} or~\citet{borlew}:
\begin{proposition} \emph{\textbf{(Subdifferential)}}
\label{prop:subdiff}
If $W = U \Diag(s) V^\top$ with $U \in \rb^{p \times m}$ and $V \in \rb^{q \times m}$ having orthonormal columns, and $s \in \rb^m$
is strictly positive, is the singular value decomposition of $W$, then $\|W\|_\ast = \sum_{i=1}^m s_i$ and the subdifferential of $\| \cdot \|_\ast$ is equal to
$$\partial \| \cdot \|_\ast(W) = 
\left\{
UV^\top + M, \mbox{ such that }  \| M \|_2 \leqslant 1, \ U^\top M = 0 \mbox{ and } MV = 0
\right\}.$$
\end{proposition}
This result can be extended to compute directional derivatives:

\begin{proposition}  \emph{\textbf{(Directional derivative)}}
\label{prop:dirder}
The directional derivative at $W=USV^\top$ is equal to:
$$
\lim_{ \varepsilon \to 0^+} \frac{ \| W + \varepsilon \D \|_\ast -
 \| W  \|_\ast} {\varepsilon} = \tr U^\top \D V + \| U_\bot^\top \D V_\bot \|_\ast,$$
where $U_\bot \in \rb^{p \times (p-m)}$ and $V_\bot \in \rb^{q \times (q-m)}$ are any orthonormal complements of $U$ and $V$.
\end{proposition}

\begin{proof}
From the subdifferential, we get the directional derivative~\citep{borlew} as
$$
\lim_{ \varepsilon \to 0^+} \frac{ \| W + \varepsilon \D \|_\ast -
 \| W  \|_\ast} {\varepsilon}  = 
\max_{V \in \partial \| \cdot \|_\ast(W)} \tr \D^\top V
$$
which exactly leads to the desired result.
\end{proof}

The final result that we use is a bit finer as it gives an upper bound on the error in the previous limit:
\begin{proposition}
Let $W = U \Diag(s) V^\top$ the ordered singular value decomposition, where $\rank(W)=r$, $s>0$ and $U_\bot$ and $V_\bot$ be orthogonal complement of $U$ and $V$; then, if $\| \Delta \|_2 \leqslant s_r/4$:
$$
\left|
\| W +   \D \|_\ast -
 \| W  \|_\ast - \tr U^\top \D V - \| U_\bot^\top \D V_\bot \|_\ast 
\right| \leqslant 16  \min \{p,q \} \frac{ s_1^2}{s_r^3} \| \D \|_2^2.
$$
\end{proposition}
\begin{proof}
The trace norm of $\| W +   \D \|_\ast$  may be divided into the sum of the $r$
largest and the sum of the remaining singular values. The sums of the remaining ones
are given through Proposition~\ref{prop:proj} by
$\| U_\bot^\top \Delta V_\bot \|_\ast$ with an error bounded by 
$  \min \{p,q \}  \frac{4}{s_r } \|  {\D} \|_2^2$. For the first $r$ singular values, we need to upperbound the second derivative of the sum of the $r$ largest eigenvalues of $\bar{W}+\bar{\D}$ with strictly positive eigengap, which leads to the given bound by using the same Cauchy residue technique described in Appendix~\ref{app:svd}.
\end{proof}

\section{Proofs}
\subsection{Proof of Lemma~\ref{lemma:cf}}
\label{app:cf}
 
We let denote $S \in  \{0,1\}^{n_x \times n_y}$ the sampling matrix; i.e.,
$S_{ij}=1$ if the pair $(i,j)$ is observed and zero otherwise. We let denote $\tilde{X}$ and $\tilde{Y}$ the data matrices. We can write
$M_k = \tilde{X}^\top  \delta_{i_k} \delta_{j_k}^\top \tilde{Y}$ and:
\BEAS
\frac{1}{n} \sum_{k=1}^n \vect(M_k) \vect(M_k)^\top
& = & \frac{1}{n} \sum_{k=1}^n (\tilde{Y} \otimes  \tilde{X} )^\top 
\vect( \delta_{i_k} \delta_{j_k}^\top )
\vect( \delta_{i_k} \delta_{j_k}^\top )^\top
 ( \tilde{Y} \otimes   \tilde{X} )  \\
 & = &   \frac{ 1}{n}
  ( \tilde{Y} \otimes  \tilde{X} )^\top  \Diag( \vect(S))
 ( \tilde{Y} \otimes  \tilde{X} ),
\EEAS
which leads to (denoting $\hS_{xx} = n_x^{-1} \tilde{X}^\top \tilde{X}$
and $\hS_{yy} = n_x^{-1} \tilde{Y}^\top \tilde{Y}$):
$$  \left(  \frac{1}{n} \sum_{k=1}^n \vect(M_k) \vect(M_k)^\top  
- \hS_{yy} \otimes \hS_{xx} \right)
=    \frac{1}{n} 
 ( \tilde{Y} \otimes  \tilde{X} )^\top  \Diag( \vect(S - n/n_x  n_y))
 ( \tilde{Y} \otimes  \tilde{X} ) .
$$
We can thus compute the squared Frobenius norm:
\BEAS
& & \left\|  \frac{1}{n} \sum_{k=1}^n \vect(M_k) \vect(M_k)^\top  - \hS_{yy} \otimes \hS_{xx} \right\|_F^2
\\ 
& = &   \frac{1}{n^2} \tr     \Diag( \vect(S - n/n_x  n_y))
 ( \tilde{Y} \tilde{Y}^\top \otimes  \tilde{X}  \tilde{X}^\top) 
 \Diag( \vect(S - n/n_x  n_y))
( \tilde{Y} \tilde{Y}^\top \otimes  \tilde{X}  \tilde{X}^\top) 
\\
& = &    \frac{1}{n^2} \sum_{i,j,i',j'} 
(S_{ij} - n/n_x  n_y
)   ( \tilde{Y} \tilde{Y}^\top \otimes  \tilde{X}  \tilde{X}^\top)_{ij,i'j'}
 (S_{i'j'} - n/n_x  n_y
  ) 
( \tilde{Y} \tilde{Y}^\top \otimes  \tilde{X}  \tilde{X}^\top)_{ij,i'j'}.
\EEAS
 We have, by properties of sampling without replacement~\citep{hoeffding}:
 \BEAS
 \E   (S_{ij} - n/n_x n_y)  (S_{i'j'} - n/n_x n_y)  & = & 
  n/n_x n_y ( 1 - n/n_x n_y) \mbox{ if } (i,j)=(i',j'),
\\
\E   (S_{ij}  - n/n_x n_y)  (S_{i'j'} - n/n_x n_y)  & = &  - 
  n/n_x n_y ( 1 - n/n_x n_y) \frac{1}{n_x n_y - 1} \mbox{ if } (i,j)\neq (i',j').
\EEAS
This implies
\begin{multline*} 
\E (
\|  \frac{1}{n} \sum_{k=1}^n \vect(M_k) \vect(M_k)^\top  - \hS_{yy} \otimes \hS_{xx} 
\|_F^2 | \tilde{X}, \tilde{Y})
\\
=  \frac{1}{n_x n_y n} \sum_{i,j}    (   \tilde{Y} \tilde{Y}^\top \otimes  \tilde{X}  \tilde{X}^\top)_{ij,ij}^2
-
\frac{1}{(n_x n_y -1)n_x n_y n}  \sum_{(i,j)\neq(i',j')}   (  \tilde{Y} \tilde{Y}^\top \otimes  \tilde{X}  \tilde{X}^\top )_{ij,i'j'} ^2 
\\
\leqslant \frac{2}{n_x n_y n} \sum_{i,j}    \| \tilde{y}_{j} \|^4 \| \tilde{x}_{i} \|^4.
 \end{multline*}
This finally implies that 
\BEAS
& & 
\E 
\left \|  \frac{1}{n} \sum_{k=1}^n \vect(M_k) \vect(M_k)^\top  - \S_{yy} \otimes \S_{xx} 
\right\|_F^2  
\\
&
\leqslant  & 
\frac{4}{ n} \sum_{i,j}    \E \| x\|^4  \E \| y\|^4 
+  2 \E \| \hS_{xx} -  \S_{xx} \|_F^2 \E \| \hS_{yy}\|_F^2  
+  2 \E \| \hS_{yy} -  \S_{yy} \|_F^2   \| \S_{xx}\|_F^2  
\\
&
\leqslant &    C \E \| x\|^4  \E \| y\|^4 \times (\frac{1}{n} +
\frac{1}{n_y} + \frac{1}{n_x}),
\EEAS
for some constant $C>0$. This implies \hypref{samplingM}.
To prove the asymptotic normality in \hypref{normalityM}, we use the martingale central limit theorem~\citep{HallHeyde}
with sequence of $\sigma$-fields $\mathcal{F}_{n, k} = \sigma(\tilde{X}, \tilde{Y}, \varepsilon_1,\dots,\varepsilon_k, (i_1,j_1),\dots,(i_k,j_k))$ for $k\leqslant n$. We consider
$\D_{n, k} = n^{-1/2} \varepsilon_{i_k j_k}  y_{j_k} \otimes x_{i_k} \in \rb^{p q}$ as the martingale difference.  We have $
\E (\D_{n, k} |\mathcal{F}_{n, k-1}) = 0$
and
$$ \E (  \D_{n, k} \D_{n, k}^\top|\mathcal{F}_{n, k-1})
= n^{-1} \sigma^2  y_{j_k} y_{j_k}^\top \otimes x_{i_k}x_{i_k}^\top,
$$
with
$\E ( \| \D_{n, k}) \|^4 ) = O(n^{-2})$ because of the finite fourth
order moments. Moreover,
$$\sum_{k=1}^n  \E ( \D_{n, k} \D_{n, k}^\top|\mathcal{F}_{n, k-1}) =  \sigma^2 \hS_{mm},$$
and thus tends in probability to $\sigma^2 \S_{yy} \otimes \S_{xx}$ because
of \hypref{samplingM}.
The assumptions of the martingale central limit theorem are met,
we have that $ \sum_{k=1}^n   \vect(\D_{n, k}) $ is asymptotically normal with mean zero 
and covariance matrix $\sigma^2 \S_{yy} \otimes \S_{xx}$, which concludes the proof.

\subsection{Proof of Proposition~\ref{prop:asympt}}
\label{app:asympt}
We may first restrict minimization over the ball $
\{ W ,\  \|W\|_\ast \leqslant \|\hS_{mm}^{-1} \hS_{Mz}\|_\ast\}$ because the optimum value is less than the value
for $W=\hS_{mm}^{-1} \hS_{Mz}$. Since this random variable is bounded in probability, we can reduce the problem to a compact set. The sequence of continuous random functions 
$ W \mapsto \frac{1}{2} \vect(W - \W)^\top \hS_{mm} \vect(W - \W)  - \tr W^\top \hS_{M\varepsilon}   + \lambda_n \| W \|_\ast $ converges pointwise in
probability to 
$  W \mapsto
\frac{1}{2} \vect(W - \W)^\top \S_{mm} \vect(W - \W)  + \lambda_0 \| W \|_\ast$ with a unique global minimum (because $\S_{mm}$ is assumed invertible). We can thus apply standard result of consistency in M-estimation~\citep{VanDerVaart,shao}.
 
 \subsection{Proof of Proposition~\ref{prop:rank2}}
 \label{app:rank2}
 We consider the result of Proposition~\ref{prop:asympt1d}: $\hat{\Delta}
  = n^{1/2}(\hat{W} - \W)$ is asymptotically normal with mean zero and
covariance $\sig \S_{mm}^{-1}$. By Proposition~\ref{prop:rankDL} in
Appendix~\ref{app:tracenorm}, if 
$ \frac{4n^{-1/2}}{s_r} \| \hD \|_2^2 < \|
\U_\bot^\top \hD \V_\bot \|_2$, then $\rank(\hat{W})>\r$.
For a random variable $\Theta$ with normal distribution with mean zero and covariance matrix  $\sig \S_{mm}^{-1}$,
we let denote $f(C) = \P (  \frac{4C^{-1/2}}{s_r} \| \Theta \|_2^2 < \|
\U_\bot^\top \Theta \V_\bot \|_2)$. By the dominated convergence theorem, $f(C)$ converges to one when $C \to \infty$. Let $\varepsilon >0$, thus there exists $C_0>0$
such that $f(C_0) > 1-\varepsilon/2$. By the asymptotic normality result,
$\P (  \frac{4C_0^{-1/2}}{s_r} \| \hD \|_2^2 < \|
\U_\bot^\top \hD \V_\bot \|_2 ) $ converges to $f(C_0)$ thus $\exists n_0>0$ such that $\forall n > n_0$, $\P (  \frac{4C_0^{-1/2}}{s_r} \| \hD \|_2^2 < \|
\U_\bot^\top \hD \V_\bot \|_2 ) > f(C_0) - \varepsilon /2 > 1 - \varepsilon$, which concludes the proof, because
$ \P (  \frac{4n^{-1/2}}{s_r} \| \hD \|_2^2 < \|
\U_\bot^\top \hD \V_\bot \|_2 )  \geqslant \P (  \frac{4C_0^{-1/2}}{s_r} \| \hD \|_2^2 < \|
\U_\bot^\top \hD \V_\bot \|_2 ) $ as soon as $n>C_0$.

\subsection{ Proof of Proposition~\ref{prop:asympt1dbis}}
\label{app:asympt1dbis}
This is the same result as~\citet{fu}, but extended to the trace norm minimization, simply using the directional derivative result
of Proposition~\ref{prop:dirder} and the epiconvergence theorem from~\citet{geyer,geyer2}. Indeed, if we denote
$V_n(\Delta) = 
\vect (\D)^\top \hS_{mm} \vect(\D) 
-  \tr \D^\top n^{1/2}\hS_{M\varepsilon} + \lambda_0 n^{1/2}
(\| \W + n^{-1/2} \Delta \|_\ast - \|\W\|_{\ast})  
$ and $V(\Delta) = \vect (\D)^\top \S_{mm} \vect(\D) 
-  \tr \D^\top A + \lambda_0  \left[\tr \U^\top \D \V + \| \U_\bot^\top \D \V_\bot \|_\ast \right]$, then for each $\Delta$, $V_n(\D)$ converges in probability to $V(\Delta)$, and $V$ is strictly convex, which implies that it has an unique global minimum; thus the epi-convergence theorem can be applied, which concludes the proof.

Note that a simpler analysis using regular tools in M-estimation leads to $\hat{W} = \W + n^{-1/2} \hD + o_p(n^{-1/2})$, where $\hD$ is the unique global minimizer of
$$
\min_{\D \in \rb^{p \times q} } \frac{1}{2}
\vect (\D)^\top \S_{mm} \vect(\D) 
-  \tr \D^\top (n^{1/2} \hS_{M\varepsilon}) + \lambda_0  \left[\tr \U^\top \D \V + \| \U_\bot^\top \D \V_\bot \|_\ast \right],
$$
i.e., we can actually take $A = n^{1/2} \hS_{M\varepsilon}$ (which is asymptotically normal with correct moments).

\subsection{ Proof of Proposition~\ref{prop:rank1}}
\label{app:rank1}
We let denote $\hD = n^{1/2}(\hat{W} - \W)$. We first show that $\lim \sup_{n \to \infty}
\P(\rank(\hat{W}) = \r) $ is smaller than the proposed limit $a$.
We consider the following events:
\BEAS
E_0 & = & \{ \rank(\hat{W}) = \r \} \\
E_1 & = &  \{ \| n^{-1/2} \hD \|_2 < \s_r/2 \} \\
E_2 & = & \left\{  \frac{4n^{-1/2}}{s_r} \| \hD \|_2^2 < \|
\U_\bot^\top \hD \V_\bot \|_2 \right\}.
\EEAS
By Proposition~\ref{prop:rankDL} in
Appendix~\ref{app:tracenorm}, we have $E_1 \cap E_2 \subset E_0^c$, and thus it suffices to show that $\P(E_1)$ tends to one, while
$\lim\sup_{n \to \infty} \P(E_2^c) \leqslant a$. The first assertion is a simple consequence
of  Proposition~\ref{prop:asympt1dbis}.

Moreover, by Proposition~\ref{prop:asympt1dbis}, $\hD$ converges in distribution to the unique global optimum $\D(A)$ of an optimization problem parameterized by a vector $A$ with normal distribution.
For a given $\eta>0$, we consider the probability
$\P( \|\U_\bot^\top \Delta(A) \V_\bot\|_2 \leqslant \eta)$. For any $A$, when $\eta$ tends to zero, the indicator function $1_{\|\U_\bot^\top \Delta(A) \V_\bot\|_2 \leqslant \eta}$ converges to $1_{\|\U_\bot^\top \Delta(A) \V_\bot\|_2 = 0}$, which is equal
to $1_{\|\Lambda(A)\|_2 \leqslant \lambda_0 }$, where
  $$
  \vect( \Lambda(A)) = \left(  (\V_\bot \otimes \U_\bot)^\top \S_{mm}^{-1} (\V_\bot \otimes \U_\bot) \right)^{-1} \!\!
\left(  (\V_\bot \otimes \U_\bot)^\top \S_{mm}^{-1} ((\V \otimes \U ) \vect(\idm) \! - \! \vect(A) ) \right).
$$
By the dominated convergence theorem, $\P( \|\U_\bot^\top \Delta(A) \V_\bot\|_2 \leqslant \eta)$ converges to $a = \P( \|\Lambda(A)\|_2 \leqslant \lambda_0)$, which is the proposed limit. This limit is in $(0,1)$   because of the normal distribution has an invertible covariance matrix and the set $\{\|\Lambda\|_2 \leqslant 1\}$ and its complement
 have   non empty interiors.
 
 Since $\hD = O_p(1)$, we can instead consider
 $ E_3 =\{  \frac{4n^{-1/2}}{s_r}M^2 < \|
\U_\bot^\top \hD \V_\bot \|_2 \} $ for a particular $M$, instead of $E_2$. Then following the same line or arguments than in Appendix~\ref{app:rank2}, we conclude that
$\lim\sup_{n \to \infty} \P(E_3^c) \leqslant a$, which concludes the first part of the proof.

\vspace*{1cm}

We now show that  $\lim \inf_{n \to \infty}
\P(\rank(\hat{W}) = \r)  \geqslant a$. A sufficient condition for rank consistency is the following:
we let denote $\hat{W} = U S V^\top$ the singular value
decomposition of $\hat{W}$ and we let denote $U_o$ and $ V_o$ the singular vectors
corresponding to all but the $\r$ largest singular values. Since we have simultaneous
singular value decompositions, a sufficient condition is that $\rank(\hat{W}) \geqslant \r$
and 
$ \left\| U_o^\top  
\left( \hS_{mm} (\hat{W}-\W) - \hS_{M\varepsilon} \right) V_o \right\|_2 < \lambda_n ( 1 - \eta)
$.
If $ \| \Lambda( n^{1/2} \hS_{M\varepsilon}) \| \leqslant  \lambda_0( 1- \eta) $, then, by Lemma~\ref{lemma:cond}, $\U_\bot^\top \Delta ( n^{1/2} \hS_{M\varepsilon}) \V_\bot = 0$, and we get, using the proof of Proposition~\ref{prop:asympt1dbis} and the notation $\hat{A} = n^{1/2} \hS_{M\varepsilon} $:
\BEAS
U_o^\top  
\left( \hS_{mm} (\hat{W}-\W) - \hS_{M\varepsilon} \right) V_o 
& =  & 
U_o^\top
 n^{-1/2}  \left(\hS_{mm} \D(\hat{A})  -\hat{A}    \right) V_o  +    o_p(n^{-1/2})  .
\EEAS
Moreover, because of regular consistency and a positive eigengap for $\W$, the projection onto the first 
$\r$ singular vectors of $\hat{W}$ converges to the projection onto the first $\r$ singular
vectors of $\W$ (see Appendix~\ref{app:svd}), which implies that the   projection onto the orthogonal is also consistent, i.e., $ U_o U_o^\top$ converges in probability to $\U_\bot \U_\bot^\top $ and
$ V_o V_o^\top$ converges in probability to $\V_\bot \V_\bot^\top $. Thus:
\BEAS
 \left\| U_o^\top  
\left( \hS_{mm} (\hat{W}-\W) - \hS_{M\varepsilon} \right) V_o \right\|_2 
&
= & 
\left\| U_o U_o^\top  
\left( \hS_{mm} (\hat{W}-\W) - \hS_{M\varepsilon} \right) V_o V_o^\top \right\|_2 \\
& = & Ä
 n^{-1/2} \|   \U_\bot\U_\bot ^\top  ( \hS_{mm} \D(\hat{A})  - \hat{A} ) \V_\bot \V_\bot^\top   \|_2 + o_p(  n^{-1/2} ) \\
 & = & 
  n^{-1/2} \|    \Lambda(A)    \|_2 + o_p( n^{-1/2}  ).
\EEAS
This implies that 
$$ \lim\inf_{n \to \infty} \left\| U_o^\top  
\left( \hS_{mm} (\hat{W}-\W) - \hS_{M\varepsilon} \right) V_o \right\|_2 < \lambda_n ( 1 - \eta) \geqslant \lim\inf_{n \to \infty} \P (  \|    \Lambda(\hat{A})    \|_2 \leqslant \lambda_0(1 - \eta))$$ which converges to $a$ when $\eta$ tends to zero, which concludes the proof.

\subsection{ Proof of Proposition~\ref{prop:asympt2}}
\label{app:asympt2} 
This is the same result as~\citet{fu}, but extended to the trace norm minimization, simply using the directional derivative result
of Proposition~\ref{prop:dirder}.
If we write $\hat{W} = \W + \lambda_n \hD$, then $\hD$ is defined as the global minimum of
\BEAS
V_n(\Delta) &  =  & \frac{1}{2} \vect(\D)^\top \hS_{mm} \vect(\D) - \lambda_n^{-1} \tr \D^\top \hS_{M\varepsilon} + \lambda_n^{-1} ( \| \W + \lambda_n \D\|_\ast -  \| \W \|_\ast)
\\
&
= &  
\frac{1}{2} \vect(\D)^\top \S_{mm} \vect(\D) + O_p( \zeta_n \|\D\|_2^2)
+ O_p(\lambda_n^{-1}n^{-1/2} ) + 
 \tr \D^\top \hS_{M\varepsilon} 
 \\
  & &  +  \tr \U^\top \D \V + \| \U_\bot^\top \D \V_\bot \|_\ast + O_p(\lambda_n \| \D\|_2^2) \\
  & = & V(\Delta) +
   O_p( \zeta_n \|\D\|_2^2)
+ O_p(\lambda_n^{-1}n^{-1/2} ) + 
   O_p(\lambda_n \| \D\|_2^2) .
\EEAS
   More precisely, if $ M \lambda_n <  \s_r/2$,
  \BEAS
  \E \!\!\! \sup_{\|\Delta\|_2 \leqslant M }
  | V_n(\Delta)  - V(\Delta) |
 & \!\!=\!\! &  \mbox{cst} \times \left(  M^2 \E \| \hS_{mm} - \S_{mm} \|_F  +  M \lambda_n^{-1} \E(  \| \hS_{M\varepsilon} \|^2)^{1/2} + \lambda_n M^2 \right)
 \\
 & = &  O( M^2 \zeta_n + M \lambda_n^{-1} n^{-1/2} + \lambda_n M^2).
 \EEAS
 Moreover, $V(\Delta)$ achieves its minimum at a bounded point $\Delta_0 \neq 0$. Thus, by Markov inequality the minimum of $V_n(\Delta)$ over the ball $\|\D \|_2 < 2 \| \D_0 \|_2$ is with probability tending to one strictly inside and is thus also the unconstrained minimum, which leads to the proposition. 
 
\subsection{ Proof of Proposition~\ref{lemma:cond}}
\label{app:cond} 
 
 The optimal $\D \in \rb^{p \times q}$
 should be such that  $\U_\bot^\top \D \V_\bot$
has low rank, where $\U_\bot \in \rb^{ p \times (p-\r)}$
and $\V_\bot \in \rb^{ q \times (q-\r)}$ are orthogonal complements of the singular
vectors $\U$ and $\V$. We now derive the condition under which the optimal $\D$ is such that 
$\U_\bot^\top \D \V_\bot$ is actually  equal to zero: we consider the minimum of
$\frac{1}{2}
\vect (\D)^\top \S_{mm} \vect(\D) 
 +  \vect(\D)^\top \vect(\U \V^\top) $ with respect to $\D$ such that 
 $\vect( \U_\bot^\top \D \V_\bot) = (\V_\bot \otimes \U_\bot)^\top \vect(\D) = 0 $. The solution of that constrained optimization problem is
 obtained through the following linear system~\citep{boyd}:
\BEQ
\label{eq:cond1}
 \left( \begin{array}{cc}
 \S_{mm}  & (\V_\bot \otimes \U_\bot) \\
 (\V_\bot \otimes \U_\bot)^\top & 0  
 \end{array}
 \right)  \left( \begin{array}{c}
 \vect(\D) \\
 \vect(\Lambda)  
 \end{array}
 \right)  = \left( \begin{array}{c}
 - \vect(\U \V^\top) \\
 0  
 \end{array}
 \right),
 \EEQ
 where $\Lambda \in \rb^{ (p-\r) \times (q-\r)}$ is the Lagrange multiplier for the equality constraint. We can solve explicitly for $\D$ and $\Lambda$ which leads to 
$$\vect(\Lambda) = \left(  (\V_\bot \otimes \U_\bot)^\top \S_{mm}^{-1} (\V_\bot \otimes \U_\bot) \right)^{-1}
\left(  (\V_\bot \otimes \U_\bot)^\top \S_{mm}^{-1} (\V \otimes \U) \vect(\idm) \right),
$$ 
and $$\vect(\D) = -\S_{mm}^{-1} \vect( \U \V^\top - \U_\bot \Lambda \V_\bot^\top).
 $$
 
 Then the minimum of the function $F(\Delta)$ in \eq{expansion} is such that
$\U_\bot^\top \D \V_\bot = 0$ (and thus equal to $\D$ defined above) if  and only if
for all $\Theta \in \rb^{p \times q}$, the directional derivative of $F$ at $\Delta$ in the direction $\Theta $ is nonnegative, i.e.:
$$\lim_{\varepsilon \to 0^+}  \frac{F(\Delta+\varepsilon \Theta) - F(\D)}{\varepsilon}\geqslant 0.$$
By Proposition~\ref{prop:dirder}, this directional derivative is equal to 
\BEAS
\tr \Theta^\top ( \S_{mm} \Delta + \U \V^\top ) + \| \U_\bot^\top \Theta \V_\bot 
\|_\ast 
& = &   \tr \Theta^\top \U_\bot \Lambda \V_\bot + \| \U_\bot^\top \Theta \V_\bot 
\|_\ast  \\
& = &   \tr  \Lambda^\top \U_\bot^\top \Theta \V_\bot + \| \U_\bot^\top \Theta \V_\bot 
\|_\ast  .
\EEAS
Thus the directional derivative is always non negative if for all $\Theta' \in \rb^{ (p-\r) \times (q-\r)}$, $
\tr     \Lambda^\top \Theta'  + \|\Theta'\|_\ast \geqslant 0 $, i.e., if and only if
$\|\Lambda\|_2 \leqslant 1$, which concludes the proof.

\subsection{ Proof of Theorem~\ref{theo:sufficient}}
\label{app:sufficient} 
Regular consistency is obtained by Corollary~\ref{prop:regular-consistency}.
We consider the problem in \eq{expansion} of Proposition~\ref{prop:asympt2},
where $\lambda_n n^{1/2} \to \infty$ and $\lambda_n \to 0$. 
Since \eq{strict} is satisfied, the solution $\D$ indeed
satisfies $\U_\bot^\top \D \V_\bot = 0$ by Lemma~\ref{lemma:cond}.

We   have  $\hat{W} = \W + \lambda_n \D + o_p(\lambda_n)$ and we now show that
the optimality conditions are satisfied with rank $\r$. From the regular consistency,
the rank of $\hat{W}$ is, with probability tending to one, larger than $\r$ (because the rank is lower semi-continuous function). We now need to show that
it is actually equal to $\r$. We let denote $\hat{W} = U S V^\top$ the singular value
decomposition of $\hat{W}$ and we let denote $U_o$ and $ V_o$ the singular vectors
corresponding to all but the $\r$ largest singular values. Since we have simultaneous
singular value decompositions, we
simply need to show that,
$ \left\| U_o^\top  
\left( \hS_{mm} (\hat{W}-\W) - \hS_{M\varepsilon} \right) V_o \right\|_2 < \lambda_n
$ with probability tending to one. We have:
\BEAS
U_o^\top  
\left( \hS_{mm} (\hat{W}-\W) - \hS_{M\varepsilon} \right) V_o 
& =  & 
U_o^\top
\left( \lambda_n \hS_{mm} \D +   o_p(\lambda_n) - O_p(n^{-1/2})  \right) V_o  \\
& = &  \lambda_n U_o^\top  
  (\S_{mm} \D ) V_o   + o_p(\lambda_n).
\EEAS
Moreover, because of regular consistency and a positive eigengap for $\W$, the projection onto the first 
$\r$ singular vectors of $\hat{W}$ converges to the projection onto the first $\r$ singular
vectors of $\W$ (see Appendix~\ref{app:svd}), which implies that the   projection onto the orthogonal is also consistent, i.e., $ U_o U_o^\top$ converges in probability to $\U_\bot \U_\bot^\top $ and
$ V_o V_o^\top$ converges in probability to $\V_\bot \V_\bot^\top $. Thus:
\BEAS
 \left\| U_o^\top  
\left( \hS_{mm} (\hat{W}-\W) - \hS_{M\varepsilon} \right) V_o \right\|_2 
&
= & 
\left\| U_o U_o^\top  
\left( \hS_{mm} (\hat{W}-\W) - \hS_{M\varepsilon} \right) V_o V_o^\top \right\|_2 \\
& = & Ä
 \lambda_n \|   \U_\bot\U_\bot ^\top ( \S_{mm} \D) \V_\bot \V_\bot^\top   \|_2 + o_p(\lambda_n) \\
 & = & 
 \lambda_n \|    \Lambda    \|_2 + o_p(\lambda_n).
\EEAS
This implies that
that the last expression is  asymptotically of magnitude strictly less than one, which concludes the proof.

\subsection{ Proof of Theorem~\ref{theo:necessary}}
\label{app:necessary} 
 We have seen earlier that if $n^{1/2} \lambda_n$ tends to zero and $\lambda_n$ tends to zero, then \eq{weak} is necessary for rank-consistency. We just have to show that there is a subsequence that does satisfy this. If $\lim\inf \lambda_n > 0 $, then we cannot have consistency (by Proposition~\ref{prop:asympt1d}), thus if we consider a subsequence, we can always assume that $\lambda_n$ tends to zero.
 
 We now consider the sequence $n^{1/2} \lambda_n$, and its accumulation points. If zero or $+\infty$ is one of them, then by Propositions~\ref{prop:rank2} and~\ref{prop:rank1}, we cannot have rank consistency. Thus, for all acccumulation points (which are finite and strictly positive), by considering a subsequence, we are in the situation where $n^{1/2} \lambda_n$ tends to $+\infty$ and $\lambda_n$ tends to zero, which implies \eq{weak}, by definition of $\Lambda$ in \eq{lambda} and Lemma~\ref{lemma:cond}.

\subsection{ Proof of Theorem~\ref{theo:adaptive}}
\label{app:adaptive} 

We let denote $U_{LS}^\r$  and $V_{LS}^\r$  the first $\r$ columns of $U_{LS}$
and $V_{LS}$ and $U_{LS}^o$  and $V_{LS}^o$ the remaining columns; we also denote 
$s_{LS}^\r$ the corresponding first $\r$ singular values and $s_{LS}^o$ the remaining singular values.
From Lemma~\ref{lemma:LS} and  results in the appendix, we get that 
$\| s_{LS}^\r - \s \|_2 = O_p(n^{-1/2})$ and $\| s_{LS}^o  \|_2 = O_p(n^{-1/2})$
and
$ \|U_{LS}^\r (U_{LS}^\r)^\top - \U \U^\top\|_2 = O_p(n^{-1/2})$ and
$ \|V_{LS}^\r (V_{LS}^\r)^\top - \V \V^\top\|_2 = O_p(n^{-1/2}).$
By writing $\hat{W}_A = \W + n^{-1/2} \hD_A$, $\hD_A $ is defined as the minimum of
$$
\frac{1}{2}\vect(\D)^\top \hS_{mm} \vect(\D) - n^{1/2} \tr \D^\top \hS_{M\varepsilon} +
 n \lambda_n \left( \| A \W B  + n^{-1/2} A \D B \|_\ast
 - \| A \W B \|_\ast
  \right).
$$
We have: 
\BEAS
A \U & = & U_{LS} \Diag(s_{LS})^{-\gamma} U_{LS}^\top \U \\
& = & 
U_{LS}^\r \Diag(s_{LS}^\r)^{-\gamma} (U_{LS}^\r)^\top \U  + U_{LS}^o \Diag(s_{LS}^o)^{-\gamma} (U_{LS}^o)^\top \U \\
 & = & \U  \Diag(\s) ^{-\gamma}
 + O_p(n^{-1/2}) + O_p(n^{-1/2} n^{\gamma/2} ) 
  \\
  & = & \U  \Diag(\s) ^{-\gamma}  +
 O_p(n^{-1/2} n^{\gamma/2} ),
 \EEAS
 and 
 \BEAS
A \U_\bot & = & U_{LS} \Diag(s_{LS})^{-\gamma} U_{LS}^\top \U_\bot \\
& = & 
U_{LS}^\r \Diag(s_{LS}^\r)^{-\gamma} (U_{LS}^\r)^\top \U_\bot  + U_{LS}^o \Diag(s_{LS}^o)^{-\gamma} (U_{LS}^o)^\top \U_\bot \\
& = & \U_\bot  \Diag(s_{LS}^o)^{-\gamma} + O_p(n^{\gamma/2-1/2}) \\
 & = &    O_p(  n^{\gamma/2} ) .
 \EEAS
Similarly we have: $ B \V = \V \Diag(\s) ^{-\gamma}  +
 O_p(n^{-1/2} n^{\gamma/2} )$ and $ B\V = O_p(  n^{\gamma/2} ) $. We can decompose any $\D \in \rb^{p \times q}$ as
 $\D = ( \U \ \U_\bot ) \left(
 \begin{array}{cc} \D_{\r\r} & \D_{\r o} \\ \D_{o \r} & \D_{oo}
 \end{array}\right)
 ( \V \ \V_\bot )^\top $.
We have assumed that $\lambda_n n^{1/2} n^{\gamma/2}$ tends to infinity. Thus,
\BIT
\item if $\U_\bot^\top \D = 0$ and $  \D \V_\bot = 0$ (i.e.,
if $\D$ is of the form $  \U \Delta_{\r\r} \V^\top$),
\BEAS
 n \lambda_n \| A \W B  + n^{-1/2} A \D B \|_\ast - \| A \W B \|_\ast
& \leqslant &  \lambda_n n^{1/2} \| A \D B \|_\ast  \\
& = &  \lambda_n n^{1/2} \| \Diag(\s)^{-\gamma} \Delta_{\r\r} \Diag(\s)^{-\gamma} 
\|_\ast + O_p( \lambda_n n^{\gamma/2}) \\
& = &  O_p( \lambda_n n^{1/2})
\EEAS
 tends to zero.
\item Otherwise,
$ n \lambda_n \| A \W B  + n^{-1/2} A \D B \|_\ast - \| A \W B \|_\ast
 $ is larger than 
 $ \lambda_n n^{1/2} \| A \D B \|_\ast - 2 \| A \W B \|_\ast$. The term
$ \| A \W B \|_\ast$ is bounded in probability because we can write
$ A \W B = \U \Diag( \s)^{1-2\gamma} \V^\top + O_p(n^{-1/2+\gamma/2})$ and $\gamma \leqslant 1$. Besides, $\lambda_n n^{1/2} \| A \D B \|_\ast$ is tending to infinity
as soons as any of $\D_{o\r}$, $\D_{\r o}$ or $\D_{\r\r}$ are different from zero. Indeed, by equivalence of finite dimensional norms 
$\lambda_n n^{1/2} \| A \D B \|_\ast $ is larger than a constant times
$\lambda_n n^{1/2} \| A \D B \|_F$, which can be decomposed in four pieces along $(\U,\U_\bot)$ and $(\V,\V_\bot)$, corresponding asymptotically to $\D_{oo}$, $\D_{o\r}$, $\D_{\r o}$ or $\D_{\r\r}$. The smallest of those terms grows faster than
$\lambda_n n^{1/2+\gamma/2}$, and thus tends to infinity.
\EIT 

 Thus, since $\S_{mm}$ is invertible,
 by the epi-convergence theorem of~\citet{geyer,geyer2},   $\hD_{A}$ converges
 in distribution  
 to the minimum of 
 $$
\frac{1}{2}\vect(\D)^\top \S_{mm} \vect(\D) - n^{1/2} \tr \D^\top \hS_{M\varepsilon}, $$
such that $\U_\bot^\top \D = 0$ and $  \D \V_\bot = 0$. This minimum has a simple asymptotic distribution, namely $\D = \U \Theta \V^\top$ and
$\Theta$ is asymptotically normal with mean zero and covariance matrix
$ \sigma^2 \left[(\V \otimes \U)^\top \S_{mm} (\V \otimes \U) \right]^{-1}
$, which leads to the consistency and the asymptotic normality.

 In order to finish the proof, we consider the optimality conditions which can be written as $ A \D B$ and $$ A^{-1} \left( \hS_{mm} \hD_A - n^{1/2}   \hS_{M\varepsilon} \right) B^{-1} $$ having simultaneous singular value decompositions with proper decays
of singular values, i.e, such that the first $\r$ are equal to $\lambda_n n^{1/2}$ and the remaining ones are less than $ \lambda_n n^{1/2}$.

From the asymptotic normality we get that $  \hS_{mm} \hD_A - n^{1/2}   \hS_{M\varepsilon} $ is $O_p(1)$, we can thus consider matrices of the form $A^{-1} \Theta B^{-1}$ where $\Theta$ is bounded, the same way we considered matrices of the form
$A \D B$. 

We have: 
\BEAS
A^{-1} \U & = & U_{LS} \Diag(s_{LS})^{\gamma} U_{LS}^\top \U \\
& = & 
U_{LS}^\r \Diag(s_{LS}^\r)^{\gamma} (U_{LS}^\r)^\top \U  + U_{LS}^o \Diag(s_{LS}^o)^{\gamma} (U_{LS}^o)^\top \U \\
 & = & \U  \Diag(\s) ^{\gamma} + O_p(n^{-1/2} ),
 \EEAS
 and 
 \BEAS
A^{-1}\U_\bot & = & U_{LS} \Diag(s_{LS})^{\gamma} U_{LS}^\top \U_\bot \\
& = & 
U_{LS}^\r \Diag(s_{LS}^\r)^{\gamma} (U_{LS}^\r)^\top \U_\bot  + U_{LS}^o \Diag(s_{LS}^o)^{\gamma} (U_{LS}^o)^\top \U_\bot \\
& = & O_p(n^{-1/2})
+\U_\bot  \Diag(s_{LS}^o)^{\gamma} ,
 \EEAS
with similar expansions for $B^{-1} \V$ and $B^{-1} \V_\bot$.
We obtain the first order expansion:
\begin{multline*}
 A^{-1} \Theta B^{-1}   =  \U  \Diag(\s) ^{\gamma}   \Theta_{\r\r}
 \Diag(\s) ^{\gamma}  \V^\top  
 + \U_\bot  \Diag(s_{LS}^o)^{\gamma}    \Theta_{o\r}\Diag(\s) ^{\gamma}  \V^\top 
  \\ +
 \U  \Diag(\s) ^{\gamma}   \Theta_{\r o}     \Diag(s_{LS}^o)^{\gamma}    
\V_\bot^\top 
+
\U_\bot  \Diag(s_{LS}^o)^{\gamma}    \Theta_{o o }  \Diag(s_{LS}^o)^{\gamma}    
\V_\bot^\top 
 \end{multline*}

Because of the regular consistency, the first term is of the order of $\lambda_n n^{1/2}$ (so that the first $\r$ singular values of $\hat{W}$ are strictly positive), while the three other terms have
norms less than $O_p(n^{-\gamma/2})$ which is less than $O_p(n^{1/2} \lambda_n)$ by assumption. This concludes the proof.

\bibliography{tracenorm}

\end{document}